\documentclass[11pt]{article}
\usepackage{amsmath}
\usepackage{amsthm}
\usepackage{amssymb}
\usepackage{algorithm}
\usepackage{color}
\usepackage[english]{babel}
\usepackage{graphicx}
\usepackage{grffile}
\usepackage{wrapfig,epsfig}
\usepackage{url}
\usepackage{color}
\usepackage{epstopdf}
\usepackage{algpseudocode}
\usepackage[T1]{fontenc}
\usepackage{bbm}
\usepackage{comment}
\usepackage{dsfont}
\usepackage{thm-restate}
\usepackage{bm}
\usepackage{tcolorbox}
\usepackage{subcaption}

\usepackage{enumitem}

\usepackage[margin=1in]{geometry}

\usepackage{mathtools}

\newtheorem{theorem}{Theorem}[section]
 
\newtheorem{lemma}[theorem]{Lemma}
\newtheorem{definition}[theorem]{Definition}

\newtheorem{fact}[theorem]{Fact}
\newtheorem{remark}[theorem]{Remark}
\newtheorem{claim}[theorem]{Claim}

\newcommand{\wh}{\widehat}
\newcommand{\wt}{\widetilde}

\newcommand{\eps}{\epsilon}

\newcommand{\R}{\mathbb{R}}
\newcommand{\F}{\mathbb{F}}

\renewcommand{\i}{\mathbf{i}}
\renewcommand{\varepsilon}{\epsilon}
\renewcommand{\tilde}{\wt}
\renewcommand{\hat}{\wh}
\renewcommand{\bar}{\overline}
\renewcommand{\eps}{\epsilon}

\newcommand{\TV}{\mathsf{TV}}

\newcommand{\trun}{\mathsf{trun}}
\newcommand{\bad}{\mathsf{bad}}

\newcommand{\KL}{\mathsf{KL}}
\newcommand{\Img}{\mathsf{Img}}
\newcommand{\lrg}{\mathsf{large}}

\newcommand{\D}{\mathcal{D}}
\newcommand{\mH}{\mathcal{H}}
\newcommand{\X}{\mathcal{X}}
\newcommand{\A}{\mathcal{A}}

\DeclareMathOperator*{\E}{{\mathbb{E}}}

\algnewcommand\algorithmicforeach{\textbf{for each}}
\algdef{S}[FOR]{ForEach}[1]{\algorithmicforeach\ #1\ \algorithmicdo}

\makeatletter
\newcommand*{\RN}[1]{\expandafter\@slowromancap\romannumeral #1@}
\makeatother

\title{Memory Bounds for Continual Learning}
\author{  
Xi Chen \\ Columbia University \\ \texttt{xichen@cs.columbia.edu}
\and Christos Papadimitriou\\ Columbia University\\  \texttt{christos@columbia.edu}
\and Binghui Peng \\ Columbia University \\ \texttt{bp2601@columbia.edu}
}

\begin{document}
\maketitle

\begin{abstract}
Continual learning, or lifelong learning, is a formidable current challenge to machine learning. It requires the learner to solve a sequence of $k$ different learning tasks, one after the other, while 
retaining its aptitude for earlier tasks; the continual learner should scale better than the obvious solution of developing and maintaining a separate learner for each of the $k$ tasks.  We embark on a complexity-theoretic study of continual learning in the PAC framework. We make novel uses of communication complexity to establish that any continual learner, even an improper one, needs memory that grows linearly with $k$, strongly suggesting that the problem is intractable.  When logarithmically many passes over the learning tasks are allowed, we provide an algorithm based on multiplicative weights update whose memory requirement scales well; we also establish that improper learning is necessary for such performance. We conjecture that these results may lead to new promising approaches to continual learning.
\end{abstract}

\clearpage
\newpage

\section{Introduction}
\label{sec:intro}
Machine learning has made dizzying empirical advances over the last decade through the powerful methodology of deep neural networks --- learning devices consisting of compositions of parametrized piecewise linear functions trained by gradient descent --- achieving results near or beyond human performance at a plethora of sophisticated learning tasks in a very broad number of domains. There are several remaining challenges for the field, however, and in order to address these theoretical work seems to be needed.  These challenges include the methodology's reliance on massive labeled data, the brittleness of the resulting learning engines to adversarial attacks, and the absence of robust explainability and the ensuing risks for applying the technique to classification tasks involving humans. This paper is about another important ability that deep neural networks appear to lack when compared to animal learners: {\em They do not seem to be able to learn continually}.

Brains interact with the environment through sensory and motor organs, and survive by learning the environment's ever changing structure and dynamics, constraints and rules, risks and opportunities. This learning activity continues throughout the animal's life: it is {\em continual and life-long.}  We learn new languages, new math and new skills and tricks, we visit new cities and we adapt to new situations, we keep correcting past mistakes and misapprehensions throughout our lives.  

It has long been observed that neural networks cannot do this \cite{mcclelland1995there}.  In fact, when a neural network is tasked with learning a sequence of data distributions, the phenomenon of {\em catastrophic interference} or {\em catastrophic forgetting} is likely to occur: soon, the performance of the device on past tasks will degrade and ultimately collapse.  This phenomenon seems to be quite robust and ubiquitous, see \cite{goodfellow2013empirical,parisi2019continual}.  This inability is important, because many applications of machine learning involve autonomous agents and robots, domains in which one needs to continually learn and adapt \cite{parisi2019continual}.

There is a torrent of current empirical work in continual/life-long learning. There is a multitude of approaches, which can be roughly categorized this way:  First, {\em regularization} of the loss function  is often used to discourage changes in learned parameters in the face of new tasks.  Second, the learning architecture (the parameter space of the device) can be designed so that it  unfolds {\em dynamically} with each new task, making new parameters available to the task and freezing the values of some already learned parameters.  Finally, there is the approach of {\em replay,} influenced by the {\em complementary learning systems} theory of biological learning \cite{mcclelland1995there}, in which previously used data are selected and revisited periodically.  Even though none of these families of approaches seems to be particularly successful currently, our positive results seem to suggest that this latter replay approach may be the most promising.

Theoretical work on PAC learning over the past four decades \cite{shalev2014understanding} has shed light on the power and limitations of learning algorithms.  The powerful tool of the VC-dimension \cite{vapnik2015uniform} captures nicely the training data requirements of the learning task, while it also serves as a useful proxy for the number of parameters required of any neural network for solving a particular PAC learning task.
The PAC theory's predictions of computational intractability, on the other hand, have been upended by the avalanche of deep nets, and the surprising aptness of gradient descent to approximate exquisitely a certain broad class of non-convex minimization problems.
{\em This paper is a complexity analysis of continual learning in the context of PAC learning.}

How can one model continual learning in the PAC framework?  Assume a sequence of $k$ learning tasks, that is, a sequence of distributions over labeled data from  corresponding hypothesis classes. Our goal is to learn these $k$ tasks sequentially, starting from the first and proceeding, and end up with a learning device that performs well on all.  That is, the device with high probability tests well in all tasks --- this can be modeled by testing at the end on a randomly chosen task from those already learned.  

Obviously, this problem can be solved by training $k$ different devices.  The interesting question is, can it be done with fewer resources, by reusing experience from past learning to achieve the next task?  {\em Is there quantity discount in continual learning?}  Our results strongly suggest, in a qualified way, a negative answer.  Our main lower bound result is for {\em improper} learning --- the more powerful kind, where the learner is allowed distributions outside the hypothesis class --- this is crucial, since deep neural networks can be considered improper. 

Our main result bounds from below the amount of {\em memory} needed for continual learning.  Memory has been identified at a recent continual learning workshop \cite{workshop2021} as one of the main bottlenecks in the practice of continual learning. If we know that there is no ``quantity discount'' in memory, this suggests that a new parameter space is needed to accommodate new learning tasks --- in effect, one essentially needs a new device for each task to avoid catastrophic forgetting.

On the positive side, 
we show that, if we are allowed to make multiple passes over the sequence of tasks, an algorithm based on the multiplicative weight update competition of several learners (inspired by the algorithm in \cite{balcan2012distributed}) requires significantly less memory to successfully learn the distributions.  At a first glance, multiple pass algorithms may seem of limited value to life-long learning --- after all, you only live once ... But in application domains such as robotics or autonomous driving, the revisiting of similar environments may be routine.  Also, the positive result does suggest that the replay strategy in continual learning explained above may achieve some less exacting goal than learning all distributions --- for example, passing the $k+1$ distributions test with some high probability.  We develop this idea further in the Discussion section.

Finally, we show an {\em exponential} separation result establishing the power of improper learning algorithms in the multi-pass case: We show that any {\em proper} learning algorithm requires at least polynomial number of passes to match the performance of the best improper learner.

\subsection{Our results}

Let $\mH$ be a hypothesis class with VC dimension $d$. Let $\X$ be its data universe and let $b=\log_2 |\X|$ be its description size (so every hypothesis $h\in \mH$ maps $\X$ to $\{0,1\}$ and every data point takes $b$ bits to store).
We are interested in 
  continual learners that make one pass
(or multiple passes) on a sequence of $k$ learning tasks
$\D_1,\ldots,\D_k$, each of which is a data distribution (i.e., supported on $\X\times \{0,1\}$).
When making a pass on $\D_i$,
  the learner can draw as many samples as it wants from $\D_i$ and uses them to update its memory.
We consider the realizable setting where there exists a hypothesis $h\in \mH$ that is consistent with all $\D_i$'s. 
The goal of a $(\eps,\delta)$-continual learner is to recover a function $h:\X\rightarrow \{0,1\}$ that,
with probability at least $1-\delta$,
  has no more than $\eps$ error
  with respect to every distribution $\D_i$.
We say a learner is \emph{proper} if the function $h$ is always a hypothesis in $\mH$, and we will consider both proper and improper learners.
We are interested in minimizing the memory of continual learners.
Formal definitions of the model can be found in Section \ref{sec:pre}.


Our main result is a lower bound, 
showing that $\Omega(kdb/\eps)$-bit  memory is necessary for any general improper continual learning algorithm.
In other words, there is no gain in learning resources from the sequential process: 

\begin{restatable}[]{theorem}{onePass}
\label{thm:lower-one}
There exists a sufficiently large constant $C_0$ such that the following holds. 
For~any $\eps \in (0, 0.01]$ and any positive integers $k,d$ and $b$ with $b\geq C_0\log (kd/\eps)$, there is a hypothesis class $\mH$ of VC dimension $2d$  over a  universe $\X$ of description size $b$
such that any $(\eps, 0.01)$-continual learner for $\mH$ over a sequence of $k+1$ tasks 
requires 
$\Omega(kdb/\eps)$ bits of memory.
\end{restatable}

Next we show that allowing a continual learner to make multiple passes can significantly reduce the amount of memory needed.
More formally, we give a 
{\em boosting} algorithm that 
  uses 
$(k/\eps)^{O(1/c)}db$ bis of memory only by making $c$ passes over the sequence of $k$ tasks. Note that, after a logarithmic number of passes, the memory resource almost matches the requirement of a single task.

\begin{restatable}[]{theorem}{multiPassAlgo}
\label{thm:upper}
Let $\eps\in (0,1/10]$ and $k, d, b, c$ be four positive integers.
Let $\mH$ be any hypothesis class of VC dimension $d$ over a data universe $\X$ of description size $b$.
There is a $c$-pass $(\eps, 0.01)$-continual learner for $\mH$ over
   sequences of $k$ tasks 
that has memory requirement and sample complexity
$$
\left(\frac{k}{\eps}\right)^{2/c}\cdot db\cdot \text{\em poly}\big(c,\log k/\eps\big)\quad\text{and}\quad
\left(\frac{k}{\eps}\right)^{4/c}\cdot \frac{k+d}{\eps}\cdot \text{\em poly}\big(c,\log kd/\eps \big),
$$
respectively.
Whenever $c \geq 4 \log (k/\eps)$,
the memory requirement becomes $db\cdot \text{\em polylog}(kd/\eps)$ and the sample complexity becomes $\tilde{O}((k+d)/\eps)$, both are optimal up to a polylogarithmic factor. 
\end{restatable}

Our continual learner in
  Theorem \ref{thm:upper} is \emph{improper}, and our last result shows that this is necessary:
We prove that any proper $c$-pass continual learner needs (roughly) $kdb/(c^3\eps)$ memory. 
This gives an exponential separation between proper and improper continual learning in terms of $c$.

\begin{restatable}[]{theorem}{multiPassLower}
\label{thm:multi-pass-lower}
There exists a sufficiently large 
  constant $C_0$ such that the following holds.
For any $\eps\in (0,1/4]$ and any positive integers $c,k,d$ and $b$ such that $d\ge C_0 c$ and $b\ge C_0\log (kd/(c\eps))$, 
there is~a hypothesis class $\mH$ of VC dimension $d$ defined over a  universe $\X$ of description size $b$
such that any $c$-pass $(\eps, 0.01)$-continual learner for $\mH$ over $2k$ tasks 
requires 
$(1/c^2)\cdot \tilde{\Omega}(kdb/(c \eps))$ memory.
\end{restatable}

\subsection{Technique overview}

\subsubsection{Lower bounds for improper learning}
We start by explaining the main ideas in our lower bound proof   of Theorem \ref{thm:lower-one}.

\paragraph{The third party.} \hspace{-0.2cm}Communication complexity is a commonly used approach for establishing lower bounds in learning theory.
Some of recent examples include \cite{dagan2019space} and \cite{kane2019communication}, both of~which make reductions from two-party communication problems.  
A natural formulation of our continual learning setting as a two-party communication problem   is to split the $k$ tasks $\D_1,\ldots,\D_k$ as inputs of Alice and Bob, and the goal is to lowerbound the one-way communication complexity for Bob to output a function that has small loss for every $\D_i$.
However, this two-party approach is doomed to fail since Alice can just send a hypothesis function $h\in \mH$ that is consistent with her tasks (by Sauer-Shelah, each function in $\mH$ can be described using roughly $db$ bits).
After receiving $h$, Bob can output a function $h'$ such that for any $x\in \X$: 
(1) If $x$ is in the support of any $\D_i$ held by Bob, $h'(x)$ is set according to $\D_i$;
(2) Otherwise, $h'(x)=h(x)$.
It is easy to verify that $h'$ achieves $0$ loss for all $\D_i$ (given that we are in the realizable setting).

The reason why the two-party approach fails is that it fails to capture the following key challenge behind improper learning with limited memory: The function returned by the learner at the end, while can be any function, must have a concise representation since it is determined by the learner's limited memory.
\emph{Our main conceptual contribution in 
  circumventing this difficulty is the introduction of a third party: Charlie as a tester.} 
While $\D_1,\ldots,\D_k$ are split between Alice and Bob as before,  
  Charlie receives a data point drawn from a certain mixture of distributions of Alice and Bob and needs to return its correct label after receiving the message from Bob. 
This three-party model captures the challenge described above since the number of bits needed to represent the function (in the message from Bob to Charlie) is counted in the communication complexity. 
We believe that the introduction of a third party in the communication problem
  may have further applications in understanding communication\hspace{0.03cm}/\hspace{0.03cm}lower bounds for improper learning problems in the future. 


\paragraph{A warm-up: Low-dimensional data.} 
\hspace{-0.2cm}We will consider $\eps$ as a small positive constant throughout the overview for convenience. 
We start with a warm-up for the case when the data universe is low-dimensional, i.e., when
  $b\coloneqq \log_2 \X=O(\log (kd ))$,
  and we aim for a memory lower bound of $\Omega(kd)$.


The basic building block is the following one-way three-party communication problem:
\begin{flushleft}\begin{itemize}
    \item Let $B=kd$. Alice holds a single element $a\in [B]$ with label $1$;
    \item Bob holds $B/2$ elements from $[B]\setminus \{a\}$ with label $0$; 
    \item Charlie receives an element that is either $a$ with probability $1/2$ 
      or drawn uniformly from Bob's elements with probability $1/2$,
      and needs to output its label.
\end{itemize}\end{flushleft}
It is not difficult to show that either Alice sends $\Omega(1)$ bits 
  of information to Bob, or Bob needs to send $\Omega(B)$ bits of information
  to Charlie, in order for Charlie to succeed with say probability $2/3$. 

Next we construct a hypothesis class
  $\mH$ and use it to define a communication problem 
  that (1) can be reduced to the continual learning of $\mH$ and (2)
  can be viewed as an indexing version of the direct sum of 
  $kd$ copies of the building block problem described above.
The data universe $\X$ of $\mH$ is $[k]\times [d]\times [B]$ (so $b=O(\log (kd))$ given that $B=kd$).
Each hypothesis $h_{i,A}:\X\rightarrow \{0,1\}$ in $\mH$ is specified by an $i\in [k]$ and $A=(a_1,\ldots,a_d)\in [B]^d$ as follows:
$$
h_{i,A}(x)=\begin{cases} 1 & \text{if $x_1\ne i$}\\
1 & \text{if $x_1=i$ and $x_3=a_{x_2}$}\\
0 & \text{otherwise (i.e., $x_1=i$ and $x_3\ne a_{x_2}$)}
\end{cases}
$$
It is easy to show that $\mH$ has VC dimension $O(d)$. In the three-party communication problem:
\begin{flushleft}\begin{itemize}
\item Alice holds $(a_{i,j}: i\in [k],j\in [d])$, where each $a_{i,j}$ is drawn from $[B]$ independently and uniformly at random,
  and she views her input as a sequence of $k$ data distributions $\D_1,\ldots,\D_k$, where $\D_i$ is uniform over $(a_{i,j}: j\in [d])$
  and all labels are $1$.
\item Bob holds $i^*$ drawn uniformly random from $[k]$ and 
  $(A_j: j\in [d])$, where each $A_j$ is a size-$(B/2)$ subset 
  drawn from $[B]\setminus \{a_{i^*,j}\}$ independently and uniformly at random. Bob views his input as a data distribution $\D_{k+1}$ that is uniform
  over $\cup_j A_j$ and all labels are $0$. (Note that $\D_1,\ldots,\D_{k+1}$ are realizable in $\mH$.)
\item Charlie holds $i^*,j^*$ and $r^*$, where $j^*$ is drawn uniformly from $[d]$ and $r^*$ is $a_{i^*,j^*}$ with probability $1/2$ and 
  is drawn uniformly from $A_{i^*,j^*}$ with probability $1/2$,
  and Charlie needs to return the label of $(i^*,j^*,r^*)$.
\end{itemize}
\end{flushleft}

On the one hand, there is an intuitive reduction from this three-party 
  communication problem to continual learning of $\mH$:
  Alice and Bob simulates a continual learner for $\mH$ over
  their $k+1$ data distributions $\D_1,\ldots,\D_{k+1}$ and Bob 
  sends the memory of the learner at the end to Charlie. 
Charlie recovers a function $h:\X\rightarrow \{0,1\}$ from
  Bob's message and uses it to return the label of $(i^*,j^*,r^*)$.
As a result, communication complexity of the three-party communication problem is a lower bound for memory requirement of the continual learning of $\mH$ (of VC dimension $O(d)$) over $k+1$ tasks.  

On the other hand, to see (informally) why $\Omega(kd)$ is a natural lower bound for the communication complexity, we note that Alice has no idea about $i^*$ and $j^*$ so the amount of information she sends about $a_{i^*,j^*}$ is expected to be $o(1)$ 
  if her message to Bob is of length $o(kd)$ (this can be made formal using 
  information-theoretic direct-sum-type arguments).
When this happens, Bob must send $\Omega(B)$ $=\Omega(kd)$ bits 
  (even if Bob knows $j^*$) in order for Charlie to have a good success probability.


\begin{figure}[!t]
    \centering
    \begin{subfigure}{0.49\textwidth}
        \includegraphics[width=\textwidth]{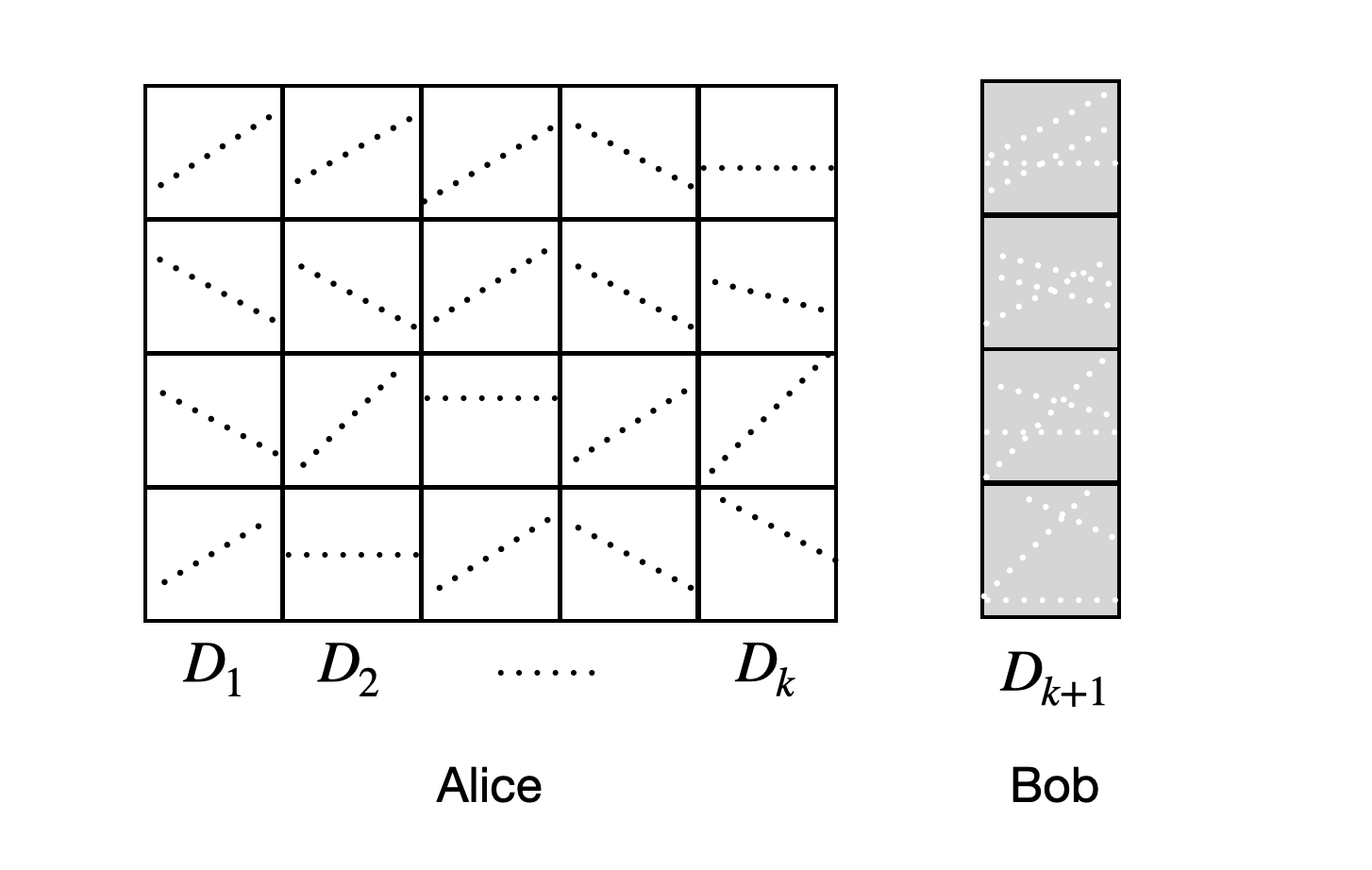}
    \end{subfigure}
    \begin{subfigure}{0.49\textwidth}
        \includegraphics[width=\textwidth]{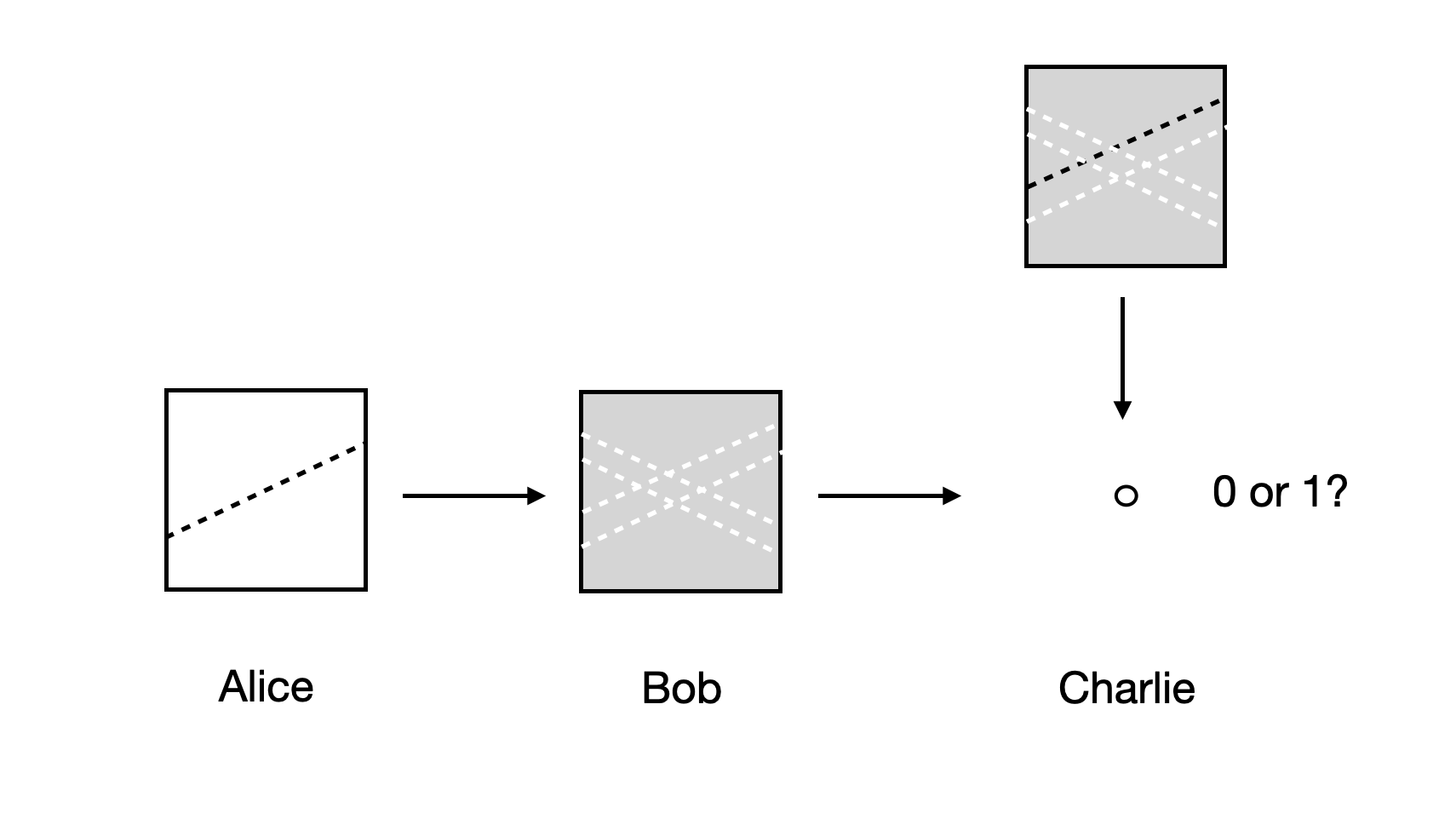}
    \end{subfigure}
    \caption{Hard instances of improper learning. A black dot indicate label-$1$ data, white background/dot means missing data, gray background means label $0$ data. The left figure illustrates the hard instance of continual learning. The right figure presents the basic building block of the asymmetric communication problem.}
    \label{fig:one-pass}
\end{figure}

\paragraph{Handling high-dimensional data.}
\hspace{-0.2cm}The key technical challenge behind the proof of Theorem \ref{thm:lower-one} comes from the high-dimensional data case, i.e., when $b\coloneqq \log_2 |\X| \gg \log (kd)$, in which case~our target lower bound is $\Omega(kdb)$ instead of $\Omega(kd)$. 
A natural approach is to change the parameter $B$ in the 
  building block communication problem of the low-dimension case
  to (roughly) $2^b$ and show that it has a communication lower bound of $\Omega(b)$.
However, this is \emph{not true} as one can randomly hash Alice's input element $a\in [B]$. For example, the players can use public randomness to partition 
  $[B]$ into $100$ random sets $B_1,\ldots,B_{100}$ of equal size and Alice sends $O(1)$ bits to Bob to reveal the $B_i$ that contains $a$.
Bob can forward $i$ to Charlie and Charlie returns $1$ if his element lies in $B_i$ and $0$ otherwise.
Given that Bob's set most likely has small overlap with $B_i$, 
  this simple $O(1)$-bit protocol succeeds with high probability,
  and it holds no matter how large $B$ is.
The above issue does not come up in the low-dimensional case
  since we only need an $\Omega(1)$ lower bound for Alice's message there.
There is indeed a deep connection with the so-called {\em representational dimension} of a hypothesis class (see \cite{beimel2019characterizing, feldman2015sample} for the exact definition).
For example, for every hypothesis class $\mH$, under uniform data distribution, there always exists an $\eps$-net of $\mH$ with size $O(\eps^{-d})$ and it is independent of the representation size $b$ of its data universe \cite{haussler1995sphere}.

We overcome this technical challenge by imposing combinatorial structures on the hypothesis class through the {\em line functions.} 
A line function $f_{a,b}$ is defined over a finite field $\F_{p}^{2}$ and  is specified by $a, b \in \F_{p}$; it labels $(x, y) \in \F_{p}^{2}$ $1$ iff it lies on the line $ax + y \equiv b \pmod p$.
The representational dimension of line functions is known to be $\Omega(\log p)$ \cite{aaronson2004limitations, feldman2015sample} which matches 
  the representation size $2\log_2 p$.
More precisely, one needs to reveal $\Omega(\log p)$ bits about $(a,b)$ in order for a second party to distinguish whether a point is drawn uniformly from the line or uniformly from $\F_p^2$.
 

However, a lower bound on the representational dimension is far from enough for our purpose. We need a lower bound for the following three-party one-way communication problem as our building block (which plays a similar role as the building block problem in the low-dimensional case):
\begin{flushleft}\begin{itemize}
    \item Alice holds a line over $\F_{p}^{2}$. Bob holds a set of points of $\F_{p}^2$ in which points of exactly $0.2p$ lines are excluded, one of which is the line of Alice.
    \item Charlie receives a query point in $\F_p^2$ that is either (1) drawn uniformly from the line of Alice with probability $1/2$, or (2) drawn from Bob's set uniformly with probability $1/2$, and the goal is to tell which case it is.
%
\end{itemize}\end{flushleft}
We prove that in order for Charlie to succeed with high probability, either Alice needs to send  a message of length $\Omega(\log p)$ to Bob, or Bob needs to send a message of length $\Omega(p^{0.5})$ to  Charlie. 
We leave the technical detail of this lower bound to Section \ref{sec:lower-one-general}, and outline the key intuition here:
(i) The advantage of the line function is that Charlie needs $\Omega(\log p)$ bits of information about Alice's line to succeed with high probability; (ii)
If Alice's message has length $o(\log p)$, then many of Bob's $\Theta(p)$ candidate lines are equally likely to be Alice's line, in which case Bob must
  send 
a large amount of information to Charlie to satisfy (i).

The rest of the proof is similar to the low-dimensional case. 
We introduce a hypothesis class $\mH$ and a three-party communication problem
  that (1) can be reduced to the continual learning of $\mH$ and (2) can be viewed as an indexing version of the direct sum of $kd$ copies of the building block problem.
See Figure \ref{fig:one-pass} for an illustration for hard instances of the communication problem.
Our lower bound is then obtained by reducing to 
  that of the building block problem via standard compression and 
  direct-sum arguments.

\subsubsection{The multi-pass algorithm}
When a continual learner can make $c$ passes over the sequence of tasks, we show that its memory requirement can be significantly reduced to (roughly) $(k/\eps)^{O(1/c)}db$. In this overview, we focus on the case when $c= \Theta(\log(k/\eps))$.

Our algorithm builds on the idea of multiplicative weights update (MWU) \cite{arora2012multiplicative} or boosting \cite{schapire1990strength}, which has been used before in similar but different contexts.
\begin{flushleft}\begin{itemize}
    \item \cite{blum2017collaborative, chen2018tight,nguyen2018improved} focus on the sample complexity of multi-task learning. 
    By viewing each distribution $\D_{i}$ ($i \in [k]$) as an expert, one can run MWU over $O(\log k)$ rounds and return the majority vote. 
Implementing their algorithm as a $O(\log k)$-pass continual learner 
reduces the memory to $\tilde{O}(db/\eps + k)$, with the optimal sample complexity of $\tilde{O}((d+k)/\eps)$. However, the linear dependence on $\eps$ in the memory upper bound is inevitable, as they treat each distribution as a black box.
    \item \cite{balcan2012distributed} focuses on the communication aspect of PAC learning over a uniform mixture of tasks and gives an approach that achieves a logarithmic dependency on $\eps$. 
    More precisely, their protocol first draws $O(d/\eps)$ samples from the uniform mixture distribution $\D$ and maintains a weighted distribution over the empirical samples set (initially set to be uniform). 
    Each round, it only transmits $O(d)$ samples and obtains a weak learner with $O(1)$ error. By updating the weighted distribution via MWU and taking a majority vote after $O(\log (1/\eps))$ rounds, the communication complexity is $\smash{\tilde{O}(db)}$.
\end{itemize}
\end{flushleft}

\paragraph{MWU with rejection sampling.} 
Our approach resembles that of \cite{balcan2012distributed}. 
The algorithm optimizes the loss on the uniform mixture  $\D=({1}/{k})\sum_{i=1}^{k}\D_{i}$ of $\D_1,\ldots,\D_k$ and obtains an $({\eps}/{k})$-accurate classifier of $\D$, which is a sufficient condition for the continual learner to succeed.
However, to execute the MWU  algorithm, the continual learner cannot draw an empirical training set as it requires $\Omega(kdb/\eps)$ memory already. 
Instead, we run the boosting algorithm over the entire distribution $\D$ and {\em implicitly} maintain the importance weight of every data point.
In particular, (1) we explicitly compute the importance weight of each distribution $\D_i$; (2) we implicitly maintain the updated weight of every data point in $\D_i$.
For (1), we can simply estimate the empirical weights via finite samples and a multiplicative approximation suffices. 
For (2), we cannot explicitly write down the updated weights of every single data point, but it suffices to sample from the updated distribution. In particular, we can run rejection sampling on the original distribution $\D_i$, and due to the MWU rule, the updated weights are within a multiplicative factor of $\tilde{O}(k^2/\eps^2)$.
Combining with a standard analysis of MWU, this yields an algorithm of memory $\tilde{O}(db + k)$ and   sample complexity  $\tilde{O}(k^2d/\eps^2)$.

\paragraph{Optimal sample complexity via truncated rejection sampling.}
The major disadvantage of the previous approach comes from the increase of sample complexity. At a first glance, it seems impossible to use $o(kd/\eps)$ samples as it is the minimum requirement to achieve $({\eps}/{k})$-error over the uniform mixture distribution $\D$.
Technically, the overhead comes from both estimation of the weights and the rejection sampling step, as the weight could blow up by a factor of $O(k^2/\eps^2)$.
The key observation is that one can iteratively truncate the top $O(\eps/c)$-quantile of the update distribution at each round. 
This brings a total loss of $O(\eps)$ for each distribution $\D_i$ after $c$-rounds (hence breaking the aforementioned lower bound) and this is affordable.
Meanwhile, the truncated data distribution is within a multiplicative factor of $O(1/\eps)$ from the original data distribution $\D_{i}$ due to a Markov bound.
Hence, both sampling and estimation get easier and introduce $O(1/\eps)$ overheads only.
Combining with a robust analysis of MWU and a streaming implementation, overall, our algorithm requires $O(\log k/\eps)$ passes, $\tilde{O}(db)$ memory and sample complexity $\tilde{O}((d+k)/\eps)$.

\subsubsection{Lower bounds for multi-pass proper learning.}
Finally, we outline the proof of Theorem \ref{thm:multi-pass-lower} which gives a memory lower bound for proper learners that make $c$ passes on $k$ tasks. 
We reduce from the classical {\em pointer chasing} problem. 
In the~two-party pointer chasing problem, Alice and Bob hold two maps $f_A, f_{B}: [n] \rightarrow [n]$, respectively.~Let $g: [n]\rightarrow [n]$ be $g(i) := f_{B}(f_{A}(i))$, and $g^{(1)}(i) = g(i)$, $g^{(\tau)}(i) := g(g^{(\tau-1)}(i))$.
The goal is for Bob to compute $g^{(c + 1)}(1)\pmod 2$.
When Alice speaks first, if $ 2c+1 $ rounds of communication are allowed, $O(c\log n)$ total communication is sufficient; 
on the other hand, when only $2c$ rounds are allowed, there is a lower bound of $\Omega(n/c - c\log n)$ \cite{nisan1991rounds, klauck2000quantum, yehudayoff2020pointer}.

The major technical challenge lies in the construction of 
  our hypothesis class which allows us to embed the pointer chasing problem in its multi-pass continual learning.


\begin{figure}[!t]
    \centering
    \includegraphics[width=\textwidth]{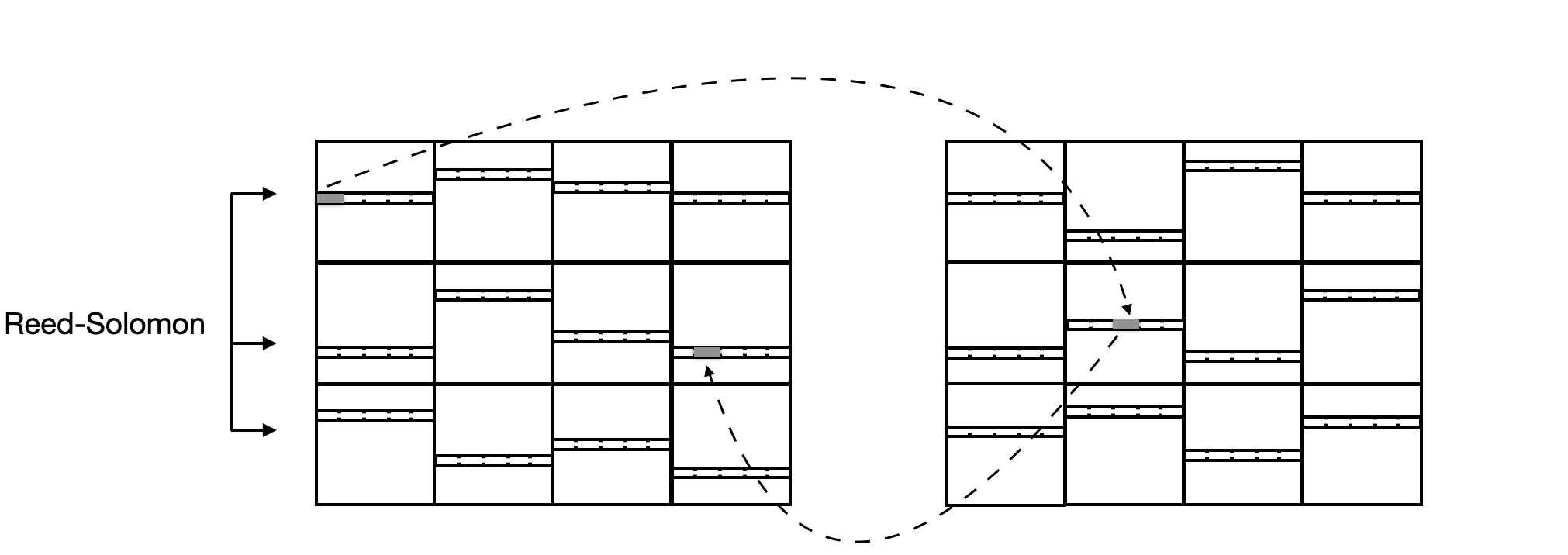}
    \caption{Hard instance of proper learner.}
    \label{fig:multi-pass}
\end{figure}

\paragraph{First attempt.} \hspace{-0.2cm}Again we ignore $\eps$ and assume it to be a small positive constant. We also assume $b\gg \log(kd)$ for simplicity. Given $c,k,d$ and $b$, our goal is to construct a hypothesis class $\mH$ with VC dimension $O(cd)$ and representation size $\tilde{O}(b)$ such that its continual learning over $2k$ tasks requires $\Omega(kdb/c^2)$ memory.
Let $n=kdb$. We start by describing the data universe $\X$ of $\mH$ and show that  input pairs $(f_A,f_B)$ with $f_A,f_B:[n]\rightarrow [n]$ of the pointer chasing problem can be viewed equivalently as a special class of Boolean functions over $\X$.

To this end, we partition $\X$ into $Y\cup Z$,
$Y$ (or $Z$) is partitioned into $k$ blocks $Y_1,\ldots,Y_k$ (or $Z_1,\ldots,Z_k$), each block $Y_i$ is  further partitioned into $d$ cells $Y_{i,1},\ldots,Y_{i,d}$ (or $Z_{i,1},\ldots,Z_{i,d}$),  
  and each cell $Y_{i,j}$ (or $Z_{i,j}$) is a disjoint copy of   $[(kdb)^b]$ (which we will view as a tuple from $[n]^b$ given that $n=kdb$).
Note that $\smash{\log_2 |\X|=O(b\log (kdb))=\tilde{O}(b)}$ as promised. 
Next we view $f_A$ as $kd$ data points in $\X$, one from each cell $Y_{i,j}$:
  $(f_A(1),\ldots,f_A(b))$ identifies an data point in $Y_{1,1}$ (viewing $Y_{1,1}$ as a copy of $[n]^b$) and in general
  $(f_A((i-1)db+(j-1)b+1),\ldots, f_A((i-1)db+(j-1)b+b))$ identifies an data point in $Y_{i,j}$.
Similarly we view $f_B$ as $kd$ data points, one from each cell $Z_{i,j}$.
Then $(f_A,f_B)$ corresponds to the function $h$ that is $1$ over these $2kd$ data points and $0$ elsewhere.

Let $\mH$ denote the class of all such Boolean functions $h$ over $\X$ (for now). 
Then ideally we would like to reduce the pointer chasing problem about $(f_A,f_B)$ to the $c$-pass continual learning of $\mH$.
To this end, we have Alice and Bob follow the following steps:
\begin{flushleft}
\begin{itemize}
    \item Given $f_A$, Alice prepares the following $k$ data distributions $\D_1,\ldots,\D_k$: each $\D_i$ is uniform over the $d$ data points she identifies in $Y_i$ (one from each $Y_{i,j}$) using $f_A$ and all are labelled $1$;
    \item Given $f_B$, Bob prepares the following $k$ data distributions $\D_{k+1},\ldots,\D_{2k}$:
    each $D_{k+i}$ is uniform over the $d$ data points he identifies in $Z_i$ using $f_B$ and all are labelled $1$;
    \item Alice and Bob run a $c$-pass proper continual learner for $\mH$ over these $2k$ data distributions using $2c-1$ rounds of communication; When the continual learner terminates, Bob uses the function $h\in \mH$ it returns (since the learner is proper) to decodes $f_A$ and then solve the pointer chasing problem.
\end{itemize}
\end{flushleft}

There are two issues with this approach.
First, the idea of having Bob recover $f_A$ at the end feels too good to be true. One would see the caveat when  analyzing the VC dimension of $\mH$ defined above,
  which turns out to be $O(k d)$ instead of the $O(cd)$ we aim for.
To resolve this issue, we refine the definition of the hypothesis class $\mH$ by requiring each $h\in \mH$ to 
  have $c+1$ special blocks $Y_{i}$, each of which contains $d$ data points labelled by $1$ and everything else labeled by $0$ as before; all data points in other non-special blocks must have label $1$, and 
  the same condition holds for $Z$.
As expected, these $2c+2$ special blocks are exactly those involved in the pointer chasing process. 
On the one hand, this refinement makes sure that the VC dimension of the new $\mH$ stays at $O(cd)$ instead of $O(kd)$;
  on the other hand, if Bob recovers $h$, he can obtain label-$1$ data points from those $c+1$ relevant blocks of $Y$ which are sufficient for him to follow the pointer chasing process.

\paragraph{Incorporating error correction codes.}
\hspace{-0.2cm}
The second issue is that the function $h\in \mH$ that 
  Bob receives from the continual learner only has the weak guarantee that its loss with respect to each $\D_i$ is small.
In particular, for any relevant block $Y_i$ of $Y$, this means that Bob is only guaranteed to recover a large portion of the $d$ label-$1$ data points in it (or equivalently, a large portion of $db$ many entries of $f_A$ in which one of them is the next pointer for Bob to follow). However, this is not good enough since the wrong data points may contain the entry of $f_A$ that Bob should follow. 


Our key idea to circumvent this issue is to incorporate {\em error correction codes} (ECC) into the construction.
To this end, the new data universe $\X$ will be similarly partitioned into blocks $Y_i$ and $Z_i$ but now each block $Y_i$ has $2d$ cells $Y_{i,1},\ldots,Y_{i,2d}$ instead, each cell being a disjoint copy of $[(kdb)^b]$. 
Taking the first block $Y_1$ as an example, Alice will still identify $d$ elements in $[(kdb)^b]$:
$$
\big(f_A(1),\ldots,f_A(b)\big),\ldots,\big(f_A((d-1)b+1),\ldots,f_A((d-1)b)\big)
$$
using $f_A$ but now she will
 apply the Reed-Solomon code to obtain  a $2d$-tuple
$(z_1,\ldots,z_{2d})$ with $z_j$ $\in [(kdb)^b]$ (for now imagine that $[(kdb)^b]$ is a finite field).
The function will have $2d$ data points labelled $1$ in $Y_1$: $z_i$ in each $Y_{1,i}$.
Modifying the hypothesis class $\mH$ this way allows Bob to fully decode the $c+1$ relevant blocks in $Y$ even if he only receives a hypothesis $h\in \mH$ from the continual learner that has small loss on every $\D_i$. 
We believe that the use of ECC may find broader applications in helping understand the complexity of learning problems in the future.


\subsection{Additional related work}
\paragraph{Continual/Lifelong learning.}
The study of continual/lifelong learning dates back to \cite{thrun1995lifelong} with the central goal of alleviating {\em catastrophic forgetting} \cite{ mccloskey1989catastrophic,mcclelland1995there}. 
There is a surge of interest from both the machine learning and neuroscience communities in recent years.
Most empirical approaches fall into three categories: regularization, replay, and dynamic architecture. 
The {\em regularization} approach penalizes the movement of hypothesis function across tasks.
The elastic weight consolidation (EWC) \cite{kirkpatrick2017overcoming} adds weighted $\ell_2$ regularization to penalize the movement of neural network parameters.
The orthogonal gradient descent (OGD)  \cite{farajtabar2020orthogonal,chaudhry2020continual} performs gradient descent orthogonal to previous directions.
The {\em replay based} approach stores training samples from old tasks and rehearses during later task.
The experience replay method \cite{rolnick2019experience} stores old experience for future use.
Instead of explicitly storing data, \cite{shin2017continual} train a simulator of past data distributions using generative adversarial network (GAN). 
\cite{van2020brain} propose a brain-inspired variant of replay method. 
The {\em dynamic architecture} approach gradually adds new components to the neural network.
The progressive network approach \cite{rusu2016progressive} allocates new subnetworks for each new task while freezing parts of the existing component.
We note that almost all current continual learning approaches suffer from a memory problem due to explicitly or implicitly retaining information from past experience.
In \cite{lopez2017gradient, chaudhry2018efficient} the authors observe that allowing multiple passes over the sequence of tasks helps improve the accuracy. We refer the interested reader to the comprehensive survey of \cite{parisi2019continual}.

In contrast to the vast experimental literature, the theoretical investigation continual/lifelong learning is limited.
Almost all existing theoretical analysis \cite{ruvolo2013ella, pentina2016lifelong,balcan2015efficient, cao2021provable,peng2022con} fall into the {\em dynamic architecture} approach: They require known task identity at inference time and maintain task specific linear classifiers.
\cite{balcan2015efficient, cao2021provable} provide sample complexity guarantee for linear representation function.
The work of \cite{knoblauch2020optimal} is closer to ours, as they aim to quantify the memory requirement of continual learning. 
However, they only show a lower bound of $\Omega(db)$ for proper learning.

\paragraph{Communication complexity of learning.}
The study of the communication aspect of learning was initiated by the work of \cite{balcan2012distributed}; they apply the idea of boosting and prove that $O(db\log(1/\eps))$ bits of communication is sufficient for learning a uniform mixture of multiple tasks, in sharp contrast with the sample complexity lower bound of $\Omega(d/\eps)$ . \cite{kane2019communication} study the two party communication problem of learning and prove a sample complexity lower bound. Their work is significantly different from ours as they consider an infinite data domain and data point transmission takes only one unit of communication. This makes their communication model {\em non-uniform} and they focus on sample complexity instead. In contrast, our work uses a finite data domain and takes into account the data description size. 
Communication complexity has also been examined for convex set disjointness \cite{braverman2021near} and optimization problems \cite{vempala2020communication}.

In addition to these works, \cite{dagan2019space} study space lower bound for linear predictor under streaming setting and prove a lower bound of $\Omega(d^2)$ for a proper learner. Other related work includes the time-space trade off \cite{raz2018fast, garg2018extractor, sharan2019memory} and memorization \cite{feldman2020does,brown2021memorization}.

\section{Preliminaries}
\label{sec:pre}

\paragraph{PAC learning.} \hspace{-0.2cm}We formulate the problem of continual learning through the PAC learning framework.
Let $\X$ be the data universe. 
A hypothesis (or a function) is a mapping $h : \X \rightarrow \{0,1\}$, and a hypothesis class $\mH$ contains a collection of functions.
A sample is a pair from $\X\times \{0,1\}$ and 
  a data distribution $\D$ is a distribution of samples, i.e., a distribution over  $\X \times \{0,1\}$.
Given a set $S$ of $n$ samples $\{(x_{i}, y_{i})\}_{i\in [n]} $, 
the {\em empirical loss} of a function $h$ with respect to $S$ is defined as
\begin{align*}
\ell_{S}(h) = \frac{1}{|S|} \sum_{(x, y) \in S} \mathbf{1}[h(x) \neq y].
\end{align*}
Given a data distribution $\D$,
the {\em loss} of $h$ with respect to $\D$ is  is defined as 
\[
\ell_{\D}(h) =  \Pr_{(x, y)\sim \D}[h(x) \neq y].
\]

\begin{definition}[PAC learning \cite{valiant1984theory}]
Given a hypothesis class $\mH$ over $\X$, an $(\eps, \delta)$-PAC learner $\mathcal{A}$ with sample complexity $n$ is a learning algorithm that, given any data distribution $\D$ over $\X\times \{0,1\}$, draws $n$ samples $S\sim \D^n$ and outputs a function $h:\X\rightarrow \{0,1\}$ such that 
\begin{align*}
\Pr_{S \sim \D^{n}, \, h = \mathcal{A}(S)}\Big[\ell_{\D}(h) \geq \min_{h^{*} \in \mH} \ell_{\D}(h^*) + \eps\Big] \leq \delta.
\end{align*}
In the {\em realizable} setting, we assume that there exists a hypothesis $h^* \in \mH$ that satisfies $\ell_{\D}(h^*) = 0$. Hence an $(\eps, \delta)$-PAC learner satisfies that
\begin{align*}
\Pr_{S \sim \D^{n}, \, h = A(S)}\big[\ell_{\D}(h) \geq \eps\big] \leq \delta.
\end{align*}
We say a learner is \emph{proper} if it always outputs a hypothesis from the class $\mH$; 
an \emph{improper} learner is allowed to output an arbitrary mapping from $\mathcal{X}$ to $\{0, 1\}$.
\end{definition}



The VC dimension measures the complexity of a function class and captures its learnability. 
\begin{definition}[VC dimension \cite{vapnik2015uniform}]
Given a set of data $S = \{x_1, \ldots, x_d\} \subseteq \X$, we say the set $S$ is shattered by function class $\mH$ if $|\mH(S)| = 2^{d}$, where $\mH(S) := \{(h(x_1), \ldots, h(x_d)) : h \in \mH\}\subseteq \{0,1\}^{d}$. The \emph{VC dimension} of $\mH$ is defined as the largest cardinality of sets shattered by $\mH$.
\end{definition}

\begin{lemma}[Sauer–Shelah Lemma \cite{sauer1972density,shelah1972combinatorial}]
Let $\mH$ be a hypothesis class with VC dimension $d$, then for any $S\subseteq \X$ with $|S| = n$, $|\mH(S)| \leq \sum_{i=0}^{d}\binom{n}{i}$. In particular,  $|\mH(S)| \leq (en/d)^{d}$ if $n\geq d$.
\end{lemma}

\begin{lemma}[Learning with ERM \cite{vapnik2015uniform}]
\label{lem:learning-erm}
Let $\mH$ be a hypothesis class with VC dimension $d$.~Given any data distribution $\D$ over $\mathcal{X}\times \{0,1\}$ that is realizable in $\mH$, the ERM learner $\mathcal{A}(S):= \{h \in \mH: $ $\ell_{S}(h) = 0\}$ is an $(\eps, \delta)$-PAC learner with sample complexity $$n = O\left(\frac{d + \log(1/\delta)}{\eps}\right).$$ 
\end{lemma}

\paragraph{Continual learning.} \hspace{-0.2cm}We next define the continual learning problem.

\begin{definition}[Continual learning]
Let $\mH$ be a hypothesis class with VC dimension $d$ over a data universe $\mathcal{X}$ of description size $b \coloneqq  \lceil\log_{2}|\X|\rceil$.
In the continual learning problem over $\mH$, a learner has sequential access to $k$ tasks, specified by $k$ data distribution $\D_1, \ldots, \D_k$ over $\X\times \{0,1\}$. We will focus on the \emph{realizable} setting: there exists a hypothesis $h\in \mH$ such that $\ell_{\D_i}(h) = 0$ for all $i \in [k]$.

A continual learner $\A$ with $m$ bits of memory is initiated with a memory state $M_0\in \{0,1\}^m$~and the learning process is divided into $k$ stages. 
During the $i$-th stage, $\A$ can draw as many samples as it wants from  $\D_i$ and use them to update the memory state from $M_{i-1}$ to $M_{i}\in \{0,1\}^m$:
\begin{align*}
S_i \sim \D_i^{*}\quad\text{and}\quad M_{i } = \A(M_{i-1}, S_i).
\end{align*}
At the end, $\A$ recovers a function
  $h:\X\rightarrow \{0,1\}$ from $M_k$.
We say $\A$ is an $(\eps,\delta)$-continual learner for $\mH$ over $k$ tasks if, with probability at least $1-\delta$,  the function $h$ satisfies 
  $\ell_{\D_i}(h) \geq \eps$ for all $i\in [k]$.
Similar to the standard PAC model,
  a learner is proper if it always outputs a hypothesis $h\in\mH$.
We are interested in understanding the memory requirement $m$ as a function of $k,d,b$ and $\eps$ for both proper and improper learners.


A $c$-pass continual learner is defined similarly, except that it can now make $c$ 
  sequential passes on $\D_1,\ldots,\D_k$.
We are interested in understanding $m$ as a function of $c,k,d,b$ and $\eps$.
\end{definition}


\begin{remark}[Sample complexity]
Our lower bounds for continual learning presented in Theorem \ref{thm:lower-one} and \ref{thm:multi-pass-lower} are for memory requirements (which make them stronger); our algorithm in Theorem \ref{thm:upper}, on the other hand, has sample complexity comparable with its memory requirement.
\end{remark}



\begin{remark}[Representation issue and bit complexity]
Our memory bounds are measured in terms of bit complexity. 
We assume a finite data universe $\X$ in which each data point can be written as $b = \lceil \log_2 |\X| \rceil$ bits. 
Due to the Sauer–Shelah Lemma, the total number of hypotheses in $\mH$ is  $2^{O(db)}$ and thus, writing down a hypothesis in $\mH$ requires at most $O(db)$ bits. 
\end{remark}

\paragraph{Notation} We write $[k] = \{1,2, \ldots, k\}$ and $[a:b] = \{a, a+1, \ldots, b\}$. We write $(i,j)\prec (i',j')$ if (1) $i<i'$ or (2) $i=i'$ and $j<j'$. For any $x, y, z\in \R$, $x = y \pm z$ means $x \in [y -z, y+z]$. The natural logarithm of $x$ is generally written as $\log x$, and the base $2$ logarithm is written as $\log_2 x$.
For any two random variables $A, B$, we use $\|A - B\|_{\TV}$ to denote the total variation distance, $\KL(A \parallel B)$ to denote the KL divergence between $A, B$ and $I(A; B)$ to denote the mutual information. The entropy function is written as $H(A)$. 
For any set $S$, we write $a\sim S$ if the element $a$ is drawn uniformly at random from $S$.

\paragraph{Prerequisites from information theory.} We provide some basic facts in information theory. The proof of Fact \ref{fact:mutual-general} is provided in the Appendix \ref{sec:app-pre} for completeness.



\begin{fact}
\label{fact:mutual-independent}
If $A_1, \ldots, A_{n}$ are totally independent, then $$\sum_{i\in [n]}I(A_i; B) \leq I(A_1, \ldots, A_n; B).$$
\end{fact}

\begin{fact}
\label{fact:mutual-general}
For any random variables $A_1, \ldots, A_n, B$, we have  $$\sum_{i\i [n]}I(A_i; B) \leq I(A_1, \ldots, A_n; B) + \sum_{i\in [n]}H(A_i) - H(A_1, \ldots, A_n).$$
\end{fact}

\begin{fact}
\label{fact:mutual-independent1}
If $A_1, \ldots, A_n$ are totally independent when conditioning on $C$, then $$\sum_{i\in [n]}I(A_i; B | C) \leq I(A_1, \ldots, A_n; B|C).$$ 
\end{fact}

\begin{fact}
\label{fact:mutual-condition}
For random variables $A, B, C, D$, if $B, D$ are independent when conditioned on $C$, then 
$$I(A; B| C) \leq I(A; B| C, D).$$
\end{fact}



We also need the following compression lemma from \cite{jain2012direct}.
\begin{lemma}[Lemma 3.4 in \cite{jain2012direct}, taking $t=1$]
\label{lem:compression}
Let $X, Y, M, R$ be random variables such that  
\begin{align*}
I(X; M | Y, R) \leq c \quad \text{and} \quad I(Y; M | X, R) = 0.
\end{align*}
There exists a one-way public-coin protocol between Alice and Bob, with inputs $X$ and $Y$~respectively, such that Alice sends a single message of size at most 
\[\frac{c + 5}{\eps'} + O\left(\log \frac{1}{\eps'}\right) 
\] 
to Bob and at the end of the protocol, Alice and Bob both hold a random variable $M'$ such that $$\big\|(X, Y, M, R) - (X, Y, M', R)\big\|_{\TV} \leq 6\eps'.$$
\end{lemma}


\def\calH{\mH}
\def\calX{\mathcal{X}}
\def\calD{\mathcal{D}}

\section{Lower bounds for continual learning}
\label{sec:lower-one-general}
We restate our memory lower bounds for continual learning.

\onePass*

Our plan for the proof of Theorem \ref{thm:lower-one} is as follows.
In Section \ref{sec:commproblem} we first introduce a three-party one-way (distributional) communication complexity problem
  parameterized by two positive integers $n,d$
  and a prime number $p$, where
  we use $f_{n,d,p}$ to denote the Boolean-valued function that the three parties would like to evaluate and 
  $\mu_{n,d,p}$ to denote the (joint) distribution of their inputs.
Given $k,d,b$ and $\eps$ that satisfy conditions
  of Theorem \ref{thm:lower-one}, we construct a hypothesis class $\mH=\mH_{k,d,b,\eps}$ 
  with VC dimension $2d+1$ and description size $b$ in Section \ref{sec:hypothesisclass}, and show via a 
  simple reduction that the memory requirement of any $(\eps,0.01)$-continual learner for  $\mH$ over $k+1$ tasks is at least
  the communication complexity of $f_{n,d,p}$ over $\mu_{n,d,p}$ with (roughly) 
  $n\approx k/\eps$ and $p\approx 2^{b/2}$.
Finally we prove that the communication   complexity of latter is at least $\Omega(\min(\sqrt{p},nd\log_2 p))$
  in the rest of the section, from which  Theorem \ref{thm:lower-one} follows directly.

\subsection{Communication problem}\label{sec:commproblem}

Let $n,d$ be two positive integers and $p$ be a prime. 
We will consider $p$ to be sufficiently large
  (given that we will set $p$ to be roughly
  $2^{b/2}$ in the reduction and that $C_0$ can be chosen to be a~sufficiently large constant in Theorem \ref{thm:lower-one}).
Let 
$t\coloneqq\lfloor 0.2p\rfloor$.
The communication problem we are interested in is one-way, distributional, and 
  has three players Alice, Bob and Charlie.
We start by describing the probability distribution $\mu_{n,d,p}$
  over inputs $X,Y$ and $Z$ of Alice, Bob and Charlie, respectively.

We follow the following procedure
  to draw $(X,Y,Z)\sim \mu_{n,d,p}$:
\begin{flushleft}\begin{enumerate}
    \item For each $i\in [n]$ and $j\in [d]$, we draw
      $a_{i,j}=(a_{i,j,1},a_{i,j,2})$ from $[p]\times [p]$ 
      independently and uniformly at random. Each $a_{i,j}$ defines
        a \emph{line} in $[p]\times [p]$, denoted by $L(a_{i,j})\subset [p]\times [p]$, which contains all points $(r_1,r_2)\in [p]\times [p]$ such that $a_{i,j,1} r_1+r_2\equiv a_{i,j,2} \pmod p$.
    \item Next for each $i\in [n]$ and $j\in [d]$, we draw a size-$t$ subset $A_{i,j}$ of $[p]\times [p]$ independently and uniformly at random but conditioning on $a_{i,j}\in A_{i,j}$.
      We write $L(A_{i,j})\subset [p]\times [p]$ to denote the union of these $t$ lines: $L(A_{i,j})\coloneqq \cup_{a\in A_{i,j}} L(a)$.
    \item Finally we draw $i^*\in [n]$, $j^*\in [d]$ and 
      a secret bit $s^*\in \{0,1\}$ independently and uniformly\\ at random.
    If $s^*=0$, we draw a point $r^*=(r^*_1,r^*_2)$ from 
      $[p]\times [p]\setminus L(A_{i^*,j^*})$ uniformly at random;
      if $s^*=1$, we draw $r^*$ from $L(a_{i^*,j^*})$ uniformly at random.
      Note that possible choices of $r^*$ are disjoint in these two cases given that $a_{i^*,j^*}\in A_{i^*,j^*}$.
    \item The input $X$ of Alice is $(a_{i,j}:i\in [n],j\in [d])$.
\\
The input $Y$ of Bob consists of  $(A_{i,j}:i\in [n],j\in [d])$ and $i^*$.\\
The input $Z$ of Charlie consists of $i^*,j^*$ and $r^*$.\\
The goal is for Charlie to output a bit that matches $s^*$. 
\end{enumerate}\end{flushleft}
Equivalently we can define a Boolean-valued function $f_{n,d,p}(X,Y,Z)$ that is set to be $0$ by default if $(X,Y,Z)$ is not in the support of $\mu_{n,d,p}$ and is set to be $s^*$ if $(X,Y,Z)$ is in the support of $\mu_{n,d,p}$ (note that $s^*$ can be uniquely determined by $(X,Y,Z)$ given that $a_{i^*,j^*}\in A_{i^*,j^*}$).
Thus the problem we are interested in is the distributional communication complexity about $f_{n,d,p}$ over $\mu_{n,d,p}$.

An $(m_1,m_2)$-bit one-way deterministic communication
  protocol is a triple $(\Pi_1,\Pi_2,$ $\Pi_3)$ of functions, where
  (1) $\Pi_1(X)\in \{0,1\}^{m_1}$ is the message from Alice to Bob,
  (2) $\Pi_2(Y,\Pi_1)\in \{0,1\}^{m_2}$ is the message from 
  Bob to Charlie, and (3)
  $\Pi_3(Z,\Pi_2)\in \{0,1\}$ is Charlie's output. 

Our main technical result in this section is the 
  following lower bound.

\begin{theorem}
\label{lem:lower-one-general}
Let $n,d$ be two positive integers and $p$ be a sufficiently large prime. 
Any $(m_1,m_2)$-bit protocol
  that fails in computing $f_{n,d,p}$
  over $\mu_{n,d,p}$ with probability at most $1/40$ must have
$$m_1=\Omega(nd\log_2 p)\quad \text{or}
\quad m_2\ge \sqrt{p}.$$
\end{theorem}

We delay the proof of the theorem to Section \ref{prooflowerbound}.

\subsection{Hypothesis class}\label{sec:hypothesisclass}

Let $k,d,b$ and $\eps$ be parameters that satisfy conditions of Theorem \ref{thm:lower-one}. 
Let $$n\coloneqq k\cdot \left\lfloor \frac{1}{100\eps} \right\rfloor
=\Theta\left(\frac{k}{\eps}\right)$$ using $\eps\le 0.01$, and
let $p$ be the largest prime with $$p\le \frac{2^{b/2}}{\sqrt{nd}}\quad \text{so that}
  \quad p=\Theta\left(\frac{2^{b/2}}{\sqrt{nd}}\right) $$
 satisfies 
  $\log_2 p = \Theta(b)$
given that $b\ge C_0 \log (kd/\eps)$
  for some sufficiently large $C_0$ in Theorem \ref{thm:lower-one}.

We describe the hypothesis class $\calH=\calH_{k,d,b,\eps}$ below. 
We start with the data universe $\calX$ of $\calH$.

\paragraph{Data universe.} Let $\mathcal{X} = [n]\times [d] \times ([p]\times [p])$.
The description size of $\X$ is no more than $b$ by our choice of $p$.
For any data point $x\in \X$, we write $x = (x_1, x_2, x_3)$ where $x_1 \in [n],$ $x_{2}\in [d]$ and $x_{3} = (x_{3, 1}, x_{3,2}) \in [p]\times [p]$.
For any $i \in [n], j \in [d]$, we refer to $\X_i = \{x\in \X: x_1 = i\}$
  as the $i$-th \emph{block} and  $\X_{i, j} = \{x\in \X: x_1 = i, x_2 = j\}$
  as the $j$-th \emph{cell} of the $i$-th block. 

\paragraph{Hypothesis class.} Each function $h$ in the hypothesis class $\mH$ is specified by an  index $i \in [n]$  and a tuple $(a_1, \ldots, a_d)\in ([p]\times [p])^{d}$ such that $h : \X \rightarrow \{0, 1\}$ is defined as
\begin{align*}
    h (x) =
    \begin{cases} 
    1 & \text{if}\ x_1 \neq i\\
    1 & \text{if}\ x_1 = i\ \text{and}\ x_3\in L(a_{x_2})\\
    0 & \text{if}\ x_1 = i\ \text{and}\ x_3\notin L(a_{x_2})
    \end{cases}
\end{align*}
In short, $h \in \mH$ has label $1$ everywhere except in the $i$-th block $\X_{i}$, where each cell $\X_{i, j}$ ($j  \in [d]$) is all $0$ except for points on the line $L(a_j)$.

A quick observation is that $\mH$ has VC dimension at most $2d$.
\begin{lemma} 
\label{lem:bounded-vc-general}
The VC dimension of $\mH$ is at most $2d$.
\end{lemma}
\begin{proof}
For any subset $S = \{x^{(1)}, \ldots, x^{(2d+1)}\}$ of size $2d+1$, we prove $\mH$ can not shatter $S$.
First, if there exists two indices $i_1, i_2 \in [2d+1]$ such that $x^{(i_1)}_1 \neq x^{(i_2)}_1$, then $\mathcal{H}$ can not shatter $(0, 0)$ on $(x^{(i_1)}, x^{(i_2)})$. 
On the other hand, suppose $x^{(1)}_1 = \ldots = x^{(2d+1)}_1 = 1$ for simplicity, then there must exists $i_1, i_2, i_3 \in [2d+1]$ such that $x^{(i_1)}, x^{(i_2)},x^{(i_3)}$ satisfy $x^{(i_1)}_2 = x^{(i_2)}_2 = x^{(i_3)}_2$, that is, there are three data points in the same cell. 
(1) If $(x^{(i_1)}_{3,1}, x^{(i_1)}_{3,2}), (x^{(i_2)}_{3,1}, x^{(i_2)}_{3,2}), (x^{(i_3)}_{3,1}, x^{(i_3)}_{3,2})$ are on a same line of $([p]\times [p])$, then $\mH$ can not shatter $(1,1, 0)$ on $(x^{(i_1)}, x^{(i_2)}, x^{(i_3)})$.
(2) If they are not on the same line, then pick any index $i_4\in [2d+1]$, we can assume $x^{(i_4)}$ is not on the same cell and same line with any two of $x^{(i_1)}, x^{(i_2)}, x^{(i_3)}$ WLOG. Then $\mH$ can not shatter $(1,1,1, 0)$ on $(x^{(i_1)}, x^{(i_2)}, x^{(i_3)}, x^{(i_4)})$. 
\end{proof}
Via a simple reduction, we show below that communication complexity lower bounds for $f_{n,d,p}$ over $\mu_{n,d,p}$ can be used to obtain memory lower bounds for the continual learning problem.
\begin{lemma}[Reduction]
\label{lem:reduction}
If there is an $(\eps,\delta)$-continual learner with an $m$-bit memory for  $\mH_{k,d,b,\eps}$ over $k+1$ tasks,
  then there is an $(m,m)$-bit protocol that fails in computing $f_{n,d,p}$ over $\mu_{n,d,p}$ with probability at most $0.01+\delta$. 
\end{lemma}

\begin{proof}
Fix any pair of inputs $X$ and $Y$ of Alice and Bob, respectively,
  where $X$ contains $(a_{i,j}: i\in [n], j \in [d])$ and $Y$ consists of $i^*$ and $(A_{i,j}: i \in [n], j \in [d])$.
Conditioning on $X$ and $Y$, Charlie's input $Z$ contains $i^*,j^*$ and $r^*$,
  where $j^*$ is uniform over $[d]$ and $r^*$ with probability $1/2$
  is uniform over $L(a_{i^*,j^*})$ and with probability $1/2$ is 
  uniform over $[p]\times [p]\setminus L(A_{i^*,j^*})$.
We describe below a communication protocol such that
  Charlie fails with probability no more than $0.01+\delta$.

Let $\X$ be the data universe of $\mH_{k,d,b,\eps}$.
Given their inputs $X$ and $Y$, respectively, Alice uses $X$ to 
a  construct a sequence of $k$ 
  data distributions $\calD_1,\ldots,\calD_k$ and Bob uses
  $Y$ to construct one data distribution $\calD_{k+1}$ over $\X\times \{0,1\}$.
We first describe these distributions and then show that 
  they are consistent with a function $h\in\mH_{k,d,b,\eps}$:
\begin{flushleft}\begin{enumerate}
\item  Let $\alpha\coloneqq \lceil 1/(100\eps)\rceil$.
For each $i\in [k]$, $\calD_{i}$ is a uniform distribution over
  data points $(x_1,x_2,x_3)$, where $x_1=(i-1)\alpha+1,\ldots,i\alpha$,
  $x_2\in [d]$ and $x_3\in L(a_{x_1,x_2})$ (so there are $\alpha d p$ points in 
  total in the support of each $\calD_i$, each with probability $1/(\alpha d p)$).
All points in $\calD_i$ are labelled $1$.
Alice can construct all these $k$ distributions using her input $X$ only.

\item Bob constructs a data distribution $\calD_{k+1}$ that is 
  uniform over $(x_1,x_2,x_3)$, where $x_1=i^*$, $x_2\in [d]$ and $x_3\in ([p]\times [p])\setminus L(A_{i^*,x_2})$,
  and all points are labelled $0$.
Given that Bob has $i^*$ and all of $A_{i^*,j}$ in $Y$, Bob can construct
  $\calD_{k+1}$ using his input $Y$ only.

\item It is clear that these data distributions are consistent with a function $h\in \calH_{k,d,b,\eps}$ that is specified by $i^*$ and $(a_{i^*,1},\ldots,a_{i^*,d})$.
\end{enumerate}\end{flushleft}

The communication protocol has Alice and Bob simulate the $(\eps,\delta)$-continual
  learner for $\calH_{k,d,b,\eps}$ where Alice passes the
  memory of the learner after going through $\calD_1,\ldots,\calD_k$ to Bob
  and Bob finally passes the memory to Charlie after going through $\calD_{k+1}$.
So both messages of Alice and Bob have $m$-bits and with 
  probability at least $1-\delta$, Charlie has in hand
  a function $h^*:\X\rightarrow \{0,1\}$ such that 
  $\ell_{\calD}(h^*)\le \eps$ for all $\calD$ in $\calD_1,\ldots,\calD_{k+1}$.
We finish the proof by showing that when this event happens, Charlie
  fails with probability no more than $0.01$ (over randomness of $Z$
  conditioning on $X$ and $Y$).
  
First, let $i'$ be the integer such that $(i'-1)\alpha+1\le i^*\le i' \alpha$. 
Then the number of data points $x=(x_1,x_2,x_3)$ with
  $x_1=i^*$, $x_2\in [d]$, $x_3\in L(a_{x_1,x_2})$ and $h^*(x)\ne h(x)$
  is no more than $\eps \cdot \alpha dp$ using 
$\ell_{\calD_{i'}}(h^*)\le \eps$.
Thus, Charlie fails because of one of these points with probability at most
$$
\frac{1}{2}\cdot \frac{\eps\alpha dp}{dp}\le \frac{0.01}{2}
$$
using $\alpha\le 1/(100\eps).$
The other error $0.01/2$ comes from that $\ell_{\calD_{k+1}}(h^*)\le \eps\le 0.01$.
\end{proof}

We are now ready to prove Theorem \ref{thm:lower-one} (assuming Theorem \ref{lem:lower-one-general} for now):

\begin{proof}[Proof of Theorem \ref{thm:lower-one}]
Given an $(\eps,0.01)$-continual learner with an $m$-bit memory for  $\mH_{k,d,b,\eps}$ over $k+1$ tasks, by Lemma \ref{lem:reduction}, there is an $(m,m)$-bit protocol that fails in computing $f_{n,d,p}$ over $\mu_{n,d,p}$ with probability at most $0.02 < \frac{1}{40}$. Hence by Theorem \ref{lem:lower-one-general}, we conclude $m \geq \Omega(\min\{ nd\log_2 p, \sqrt{p}\}) = \Omega(kdb/\eps)$. 
\end{proof}

\subsection{Communication lower bound}\label{prooflowerbound}

We prove the communication lower bound in Theorem \ref{lem:lower-one-general}. 
The proof proceeds as follows. First we prove a lower bound for the
  communication problem over one single cell in Section \ref{sec:singlecell}.
Then we follow a standard direct sum argument to extend it to 
  a lower bound for $f_{n,d,p}$ over $\mu_{n,d,p}$ in Section \ref{sec:directsum}.

\subsubsection{Communication problem over one cell}\label{sec:singlecell}

Let $p$ be a sufficiently large prime.
We first focus on the following three-party one-way distributional  communication problem over one single cell $[p]\times [p]$. 
The inputs $(a^*,A^*,r^*)$ of the three players Alice, Bob and Charlie are drawn from a distribution $\mu_p$ as follows: (1) Alice's input $a^*$ is drawn from $[p]\times [p]$ uniformly at random, which defines a line $L(a^*)$ of $p$ points;
(2) Bob's input $A^*$ is a size-$t$ subset of $[p]\times [p]$ drawn uniformly at random conditioning on $a^*\in A^*$, and we write $L(A^*)$ to denote the union of $L(a)$ over $a\in A^*$;
(3) With probability $1/2$, Charlie's input $r^*\in [p]\times [p]$ is drawn uniformly at random from $L(a^*)$; with probability $1/2$, $r^*$ is drawn uniformly at random from $([p]\times [p])\setminus L(A^*)$.
The function $f_p(a^*,A^*,r^*)$ that Charlie needs to compute is set to be $1$ if $r^*\in L(a)$ and $0$ if $r^*\in ([p]\times [p])\setminus L(A^*)$.


Our goal is to prove the following lower bound for computing $f_p$ over $\mu_p$:
\begin{lemma}
\label{lem:single}
Let $\Pi=(\Pi_1,\Pi_2,\Pi_3)$ be an $(m_1,m_2)$-bit one-way deterministic communication protocol.
If $m_1\le 0.05\log_2 p$ and $m_2\le \sqrt{p}$, then Charlie has error probability at least $9/200$.
\end{lemma}

In the proof we use $\Pi_1$ (or $\Pi_2$) to denote the  random variable as the $m_1$-bit
  (or $m_2$-bit) message from Alice to Bob
  (or from Bob to Charlie, respectively).
Lemma \ref{lem:single} is a direct corollary of Lemma \ref{lem:step3} and Lemma \ref{lem:mutual-information-bound} below.
In Lemma \ref{lem:step3} we show that our construction makes sure that in~order for Charlie to succeed with high probability, he must receive at least $(1-o(1))\log_2 p$ bits of information regarding the hidden line $a$ from $\Pi_2$.
Formally, we have

\begin{lemma}
\label{lem:step3}
For any communication protocol $\Pi = (\Pi_1, \Pi_2,\Pi_3)$, we have
\begin{align}
    \Pr_{(a^*,A^*,r^*)\sim \mu_p}
    \Big[\text{$\Pi$\ succeeds on\ $(a^*,A^*,r^*)$}\Big]
    \leq \frac{19}{20} + \frac{1}{20}\cdot\frac{I(a^*; \Pi_{2})+1}{\log_2 p}.
    \label{eq:step3}
\end{align}
\end{lemma}

We next prove the mutual information $I(a^*; \Pi_2)$ is small for any deterministic communication protocol $\Pi$ with small amount of communication. Formally
\begin{lemma}
\label{lem:mutual-information-bound}
Let $\Pi=(\Pi_1,\Pi_2,\Pi_3)$ be an $(m_1,m_2)$-bit one-way deterministic communication protocol.
If $m_1\le 0.05\log_2 p$ and $m_2\le \sqrt{p}$,
then we have $I(a^*; \Pi_2) < 0.08\log_2 p$.
\end{lemma}

Combining Lemma \ref{lem:step3} and Lemma \ref{lem:mutual-information-bound}, we conclude the proof of Lemma \ref{lem:single}.

\def\hPi{\hat{\Pi}}

\begin{proof}[Proof of Lemma \ref{lem:step3}]
Fixing a message $\Pi_2 = \hat{\Pi}_2$   from Bob,
  we upper bound the probability that 
  $\Pi$ succeeds on $(a^*,A^*,r^*)\sim \mu_p$ conditioning on $\Pi_2=\hPi_2$.
For convenience, we write $\mu_{p,\hPi_2}$ to denote this conditional distribution.  
  
With $\hPi_2$ fixed, we can assume that Charlie uses a subset of \emph{points} $H\subseteq [p]\times[p]$ to finish the protocol, i.e., after receiving  $\smash{\hat{\Pi}_2}$ from Bob,
  Charlie returns $1$ on input $r$
  if $r\in H$ and returns $0$ if $r\notin H$.
We further define a set of \emph{lines} $S$ from $H$: 
$a\in [p]\times [p]$ is in $S$ iff at least
  $90\%$ of points in $L(a)$ lie in $H$.
We divide into two cases:

\vspace{+2mm}
{\bf\noindent Case 1:} $|S|\geq p$.
Let $a_1,\ldots,a_p$ be any $p$ lines in $S$. We have 
$$
\left(\bigcup_{i\in [p]} L(a_i)\right)\bigcap H\ge
\sum_{i\in [p]} \left| \big(L(a_i)\cap H\big)
\setminus \bigcup_{j\in [i-1]} L(a_j)\right|\ge \sum_{i\in [p]} \max\big(0.9p-(i-1),0\big)>0.4p^2.
$$
On the other hand, fixing any 
  inputs $\hat{a}$ and $\hat{A}$ of Alice and Bob such that Bob sends $\hPi_2$, we have 
$$
\left|([p]\times [p])\setminus L(\hat{A})\right|\ge p^2-pt=0.8p^2 
$$
and the correct answer for all these points are $0$.
As a result, we have that when $(a^*,A^*,r^*)\sim \mu_p$ conditioning on $a^*=\hat{a}$ and 
  $A^*=\hat{A}$, 
  the probability that $\Pi$ errors is at least
$$
\frac{1}{2}\cdot \frac{0.8p^2+0.4p^2-p^2}{p^2}=0.1.
$$
This implies that 
$\Pi$ succeeds on $(a^*,A^*,r^*)\sim\mu_{p,\hPi_2}$ with probability at most $0.9$.

\vspace{+2mm}
{\bf\noindent Case 2:}   $|S| < p$. 
We show that the probability that $\Pi$ succeeds on
  $(a^*,A^*,r^*)\sim \mu_{p,\hPi_2}$ is at most
\begin{align}\label{eq:hehe}
\frac{19}{20} + \frac{1}{20} \cdot \frac{2\log_2 p - H(a^*\mid \Pi_{2} = \hat{\Pi}_{2})+1}{\log_2 p}.
\end{align}
To see this, 
  we examine the distribution of 
  $a^*$ conditioning on $\Pi_2=\hPi_2$ and write $q$ to denote the probability that $a^*\in S$.
Using $|S|\le p$, we have
\begin{align*}
H(a^*\mid \Pi_2=\hPi_2)&\le 
q \cdot \log_2\left(\frac{p}{q}\right)+(1-q)\cdot \log_2 \left(\frac{p^2-p}{1-q}\right)\\
&\le q\cdot\log_2\left(\frac{1}{q}\right)+(1-q)\cdot \log_2\left(\frac{1}{1-q}\right)
+q\log_2 p+(1-q)\log_2 p^2\\
&\le 1+(2-q)\cdot \log_2 p.
\end{align*}
Hence, we have
\begin{align}
\label{eq:correct-prob}
    q \leq \frac{2\log_2 p - H(a^*\mid \Pi_2 = \hat{\Pi}_{2}) + 1}{\log_2 p}.
\end{align}
Given that every line not in $S$ has at least $10\%$ of points not in $H$,
  $\Pi$ errors on $(a^*,A^*,r^*)\sim \mu_{p,\hPi_2}$ with probability at least
$$
\frac{1}{2}\cdot (1-q)\cdot \frac{1}{10}\ge 
\frac{1}{20}-\frac{1}{20}\cdot 
\frac{2\log_2 p-H(a^*\mid \Pi_2=\hPi_2)+1}{\log_2 p}
$$
from which (\ref{eq:hehe}) follows.\vspace{3mm} 


Combining the two cases
  (and noting that $0.9<19/20$), we have
$$
\Pr_{(a^*,A^*,r^*)\sim \mu_{p,\hPi_2}}
    \Big[\text{$\Pi$\ succeeds on\ $(a^*,A^*,r^*)$}\Big]
    \le 
     \frac{19}{20} + \frac{1}{20}\cdot 
     \frac{2\log_2 p-H(a^*\mid \Pi_2=\hPi_2)+1}{\log_2 p}.
$$
As a result, using $H(a^*)=2\log_2 p$ we have
\begin{align*}
&\hspace{-1cm}\Pr_{(a^*,A^*,r^*)\sim \mu_p}
    \Big[\text{$\Pi$\ succeeds on\ $(a^*,A^*,r^*)$}\Big]\\[0.5ex]
    &= \sum_{\hPi_2}
    \Pr\big[\Pi_2=\hPi_2]\cdot 
    \Pr_{(a^*,A^*,r^*)\sim \mu_{p,\hPi_2}}
    \Big[\text{$\Pi$\ succeeds on\ $(a^*,A^*,r^*)$}\Big]\\
    &\le \sum_{\hPi_2} \Pr\big[\Pi_2=\hPi_2]\cdot \left(\frac{19}{20} + \frac{1}{20}\cdot 
     \frac{2\log_2 p-H(a^*\mid \Pi_2=\hPi_2)+1}{\log_2 p}\right)\\[0.3ex]
     &= \frac{19}{20} + \frac{1}{20}\cdot 
     \frac{I(a^*;\Pi_2)+1}{\log_2 p}.
\end{align*}
We conclude the proof.
\end{proof}

\begin{proof}[Proof of Lemma \ref{lem:mutual-information-bound}]
For convenience, we denote $m_1 = \gamma \log_2 p$ with $\gamma \le 0.05$.
For any realization of Alice's message $\Pi_{1} = \hat{\Pi}_1$, define
\begin{align*}
\Img(\hat{\Pi}_{1}) = \big\{a\in [p]\times [p]: \Pi_1(a) = \hat{\Pi}_A\big\}.
\end{align*}
Next we write $\Pi_{\lrg}$ to denote the set of $\hPi_1$ such that 
  $\Img(\hPi_1)$ is large: 
\begin{align*}
\Pi_{\lrg} = \left\{\hat{\Pi}_{1} : \left|\Img(\hPi_1)\right| \geq \frac{p^{2-\gamma}}{100}\right\}.
\end{align*}
First we show that to bound $I(a^*;\Pi_{2})$, it suffices to bound 
  $I(a^*;\Pi_2 \mid \Pi_1=\hPi_1)$
  for $\hPi_1\in \Pi_{\lrg}$:
\begin{align}
I(a^*;\Pi_2) \leq&~ 
 I(a^*;\Pi_1\Pi_2)\notag\\
 \leq&~ I(a^*;\Pi_2 \mid  \Pi_1) +\gamma \log_2 p\notag\\
= &~ \sum_{\hat{\Pi}_{1} \in \Pi_{\lrg}}\Pr\big[\Pi_{1} = \hat{\Pi}_1\big]\cdot I(a^*;\Pi_{2} \mid \Pi_{1} = \hat{\Pi}_1)\notag \\
&~ \hspace{0.5cm}+ \sum_{\hat{\Pi}_{1} \notin \Pi_{\lrg}}\Pr\big[\Pi_{1} = \hat{\Pi}_1\big]\cdot I(a^*;\Pi_{2} \mid \Pi_{1} = \hat{\Pi}_1)  + \gamma \log_2 p\notag\\
\leq & \sum_{\hat{\Pi}_{1} \in \Pi_{\lrg}}\Pr\big[\Pi_{1} = \hat{\Pi}_1\big]\cdot I(a^*;\Pi_{2} \mid \Pi_{1} = \hat{\Pi}_1) + \frac{1}{p^2}\cdot\frac{p^{2-\gamma}}{100}\cdot p^{\gamma}\cdot 2\log_2 p + \gamma \log_2 p\notag\\
= &~\sum_{\hat{\Pi}_{1} \in \Pi_{\lrg}}\Pr\big[\Pi_{1} = \hat{\Pi}_1\big]\cdot I(a^*;\Pi_{2} \mid \Pi_{1} = \hat{\Pi}_1) + (\gamma + 0.02)\log_2 p.
\label{eq:mt-2}
\end{align}
It suffices to show that for any message $\hat{\Pi}_{1} \in \Pi_{\lrg}$, we have
\begin{align}
   I(a^*;\Pi_{2} \mid \Pi_{1} = \hat{\Pi}_1)= o_p(1).  \label{eq:mt-1}
\end{align}

\def\hc{c\hspace{0.03cm}}

Fixing a message $\hPi_1\in \Pi_{\lrg}$, let's take a pause and review the distribution of $(a^*,A^*,r^*)\sim \mu_{p}$ conditioning
  on $\Pi_1=\hPi_1$.
In particular, let's denote the distribution of $(a^*,A^*)$ as $\mu_{p,\hPi_1}$
  because~we don't care about $r^*$ in this lemma; 
  $(a^*,A^*)\sim \mu_{p,\hPi_1}$ is drawn as follows:
\begin{enumerate}
    \item Sample a line $a^*$ from $\Img(\hat{\Pi}_1)$ uniformly at random;
    \item Sample a set of $t-1$ lines from $[p]\times [p]\setminus \{a^*\}$ uniformly at random and add $a^*$ to form $A^*$.
\end{enumerate}
Let $S^*=A^*\cap \Img(\hPi_1)$ and $T^*=A^*\setminus S^*$. 
We prove that $|S^*|$ is large with high probability, when $(a^*,A^*)\sim \mu_{p,\hPi_1}$. 
Formally, after drawing $a^*$, we draw a sequence of $t-1$ lines 
  from $[p]\times [p]\setminus \{a^*\}$ without replacements.
The probability that the $i$th line is in $\smash{\Img(\hPi_1)}$ is at least
$$
\frac{0.01\cdot p^{2-\gamma}-i}{p^2}\ge \frac{p^{ -\gamma}}{101}
$$
using $t=0.2p$.
Using Chernoff bound, we have
\begin{align}
\label{eq:chern}
\Pr_{(a^*,A^*)\sim \mu_{p,\hPi_1}}\left[|S^*| < \frac{p^{1-\gamma}}{200}\right]  
\leq \exp (-\Omega(p^{1-\gamma}) ).
\end{align}
Denote $C = p^{1-\gamma}/200$.
We show that to prove (\ref{eq:mt-1}), it suffices to prove that 
\begin{equation}\label{eq:hehe2}
I\big(a^*;\Pi_2\mid \Pi_1=\hPi_1, |S^*|=\hc\big)=o_p(1)
\end{equation}
  for every $c\ge C$.
This is because
\begin{align}
    I(a^*; \Pi_{2} \mid \Pi_{1} = \hat{\Pi}_1) \leq &~ I(a^*; \Pi_{2} \mid \Pi_{1} = \hat{\Pi}_1, |S^*|)\notag\\
    = &~ \sum_{\hc \geq C} \Pr\left[|S^*| = \hc \mid \Pi_1 = \hat{\Pi}_1\right]\cdot I\big(a^*; \Pi_{2} \mid \Pi_{1} = \hat{\Pi}_1, |S^*| = \hc\big)\notag\\
    &~ \hspace{0.5cm}+ \sum_{\hc < C} \Pr\left[|S^*| = \hc \mid \Pi_1 = \hat{\Pi}_1\right]\cdot I\big(a^*; \Pi_{1} \mid \Pi_{1} = \hat{\Pi}_1, |S^*| = \hc\big) \notag \\
    \leq &~ \sum_{\hc \geq C} \Pr\left[|S^*| = \hc \mid \Pi_1 = \hat{\Pi}_1\right]\cdot I\big(a^*; \Pi_2\mid \Pi_{1} = \hat{\Pi}_1, |S^*| = \hc\big) + o(1).\label{eq:mt-3}
\end{align}
The first step holds as $|S^*|$ is independent of $a^*$ when conditioning on $\Pi_{1} = \hat{\Pi}_1$ (see Fact \ref{fact:mutual-condition}).
The last step follows from \eqref{eq:chern} and that $a^*$ takes values in $[p]\times [p]$.
 
\def\hi{i\hspace{0.03cm}}
 
For a fixed $\hat{\Pi}_1 \in \Pi_\lrg$ and a fixed $\hc \geq C$, 
  we prove (\ref{eq:hehe2}) in the rest of the proof.
Let's again take a pause and think about the distribution of 
  $a^*,S^*,T^*$ and $\Pi_2$ when conditioning on $\smash{\Pi_1=\hPi_1}$ and 
  $|S^*|=\hc$.
Equivalently, $a^*,S^*,T^*$ and $\Pi_2$ can be drawn as follows:
\begin{enumerate}
    \item Draw a sequence of $\hc$ lines $a_1,\ldots,a_{\hc}$ 
      uniformly from $\Img(\hPi_1)$ without replacements; 
    \item Draw a subset $T^*$ of size $t-\hc$ from $([p]\times [p])\setminus 
      \Img(\hPi_1)$ uniformly at random;
    \item Draw $i^*\in [\hc]$ uniformly at random; and
    \item Set $a^*=a_{i^*}$, $S^*=\{a_1,\ldots,a_{\hc}\}$ and $\Pi_2
      = \Pi_2(S^*\cup T^*, \hPi_1)$.
\end{enumerate}
From this description we have that
\begin{align}
    I\big(a^*; \Pi_{2} \mid \Pi_{1} = \hat{\Pi}_1, |S^*| = \hc\big) \leq &~ I\big(a^*; \Pi_{2} \mid \Pi_{1} = \hat{\Pi}_1, |S^*| = \hc, i^*\big)\notag\\
    = &~ \frac{1}{\hc}\sum_{\hi\in [\hc]}I\big(a_{ \hi }; \Pi_{2} \mid \Pi_{1} = \hat{\Pi}_1, |S^*| = \hc, i^*= \hi)\notag \\
    = &~ \frac{1}{\hc}\sum_{\hi\in [\hc]}I\big(a_{ \hi}; \Pi_{2} \mid \Pi_{1} = \hat{\Pi}_1, |S^*|=\hc\big)\label{eq:mt-4}
\end{align}
The first step holds since conditioning on $\Pi_1 = \hat{\Pi}_1$ and $|S^*| = \hc$, the index $i^*$ and the line $a^*$ are independent (no matter what $i^*$ is, $a^*$ is always uniform over $\Img(\hPi_1)$);
the last step holds since 
$i^*$ is independent of $(a_{1}, \ldots, a_{\hc}, \Pi_{2})$ so whether knowing $i^*=\hi$ or not does not 
  affect the distribution of $(a_{\hi},\Pi_2)$.

Let $N=|\Img(\hPi_1)|\ge p^{2-\gamma}/100$.
To bound the RHS of \eqref{eq:mt-4}, we have
\begin{align}
    &~\hspace{-1.5cm}\sum_{\hi\in [\hc]}I\big(a_{\hi}; \Pi_{2} \mid \Pi_{1} = \hat{\Pi}_1, |S^*|=c\big)\notag\\[-2ex]
    \leq &~ I\big(a_{1}\ldots a_{\hc}; \Pi_{2 } \mid \Pi_{1} = \hat{\Pi}_1, |S^*|=c\big) \notag \\[1ex]
    &~ \hspace{1cm}+ \sum_{\hi\in [\hc]}H\big(a_i\mid \Pi_{1} = \hat{\Pi}_1, |S^*|=c\big)- H\big(a_1 \ldots a_{\hc}\mid \Pi_{1} = \hat{\Pi}_1, |S^*|=c\big)\notag\\
    \leq &~ m_2 + \hc\cdot \log_2 N- \sum_{\hi\in [\hc]}\log_2 (N - \hi+1)\notag\\
    = &~ m_2 + \sum_{\hi\in [\hc]}  \log_2\left( \frac{N}{N - \hi+1}\right)\notag
\end{align}
where the first step holds due to Fact \ref{fact:mutual-general}.
Using $c\le t=0.2p$ and $N\ge p^{2-\gamma}/100$, we have 
\begin{align*}
    \sum_{\hi\in [\hc]}  \log_2 \left(\frac{N}{N - \hi+1} \right) \leq 2\sum_{\hi\in [t]}\frac{\hi}{N - \hi} \leq 2\sum_{\hi\in [t]}\frac{t}{N - t} \leq \frac{2t^2}{N - t} \leq 100 p^{\gamma},
\end{align*}
where we used $\log_2(1+x) \leq 2x$ when $x\ge 0$. 
Combining with \eqref{eq:mt-4} and $c\ge C$, we have 
\begin{align}
    I\big(a^*; \Pi_2\mid \Pi_{1} = \hat{\Pi}_1, |S^*| = \hc\big) \leq \frac{m_2 + 100p^{\gamma}}{\hc} 
    =o_p(1),  
\end{align}
where we used the assumption that $\gamma \leq 0.05$ and $m_2\leq \sqrt{p}$.
This finishes the proof of (\ref{eq:hehe2}).
\end{proof}

\subsubsection{Proof of Theorem \ref{lem:lower-one-general} via a direct sum argument}\label{sec:directsum}

Finally we apply the compression protocol from \cite{jain2012direct} (see Lemma \ref{lem:compression}) and 
  a direct sum argument similar to that of \cite{barak2013compress,braverman2011information,braverman2015interactive} 
  to reduce the lower bound of Theorem \ref{lem:lower-one-general} to that of
  Lemma \ref{lem:single} for one cell.
  
\begin{proof}[Proof of Theorem \ref{lem:lower-one-general}]
Recall in the one-cell problem, Alice receives $a\in [p]\times [p]$,
  Bob receives a size-$t$ subset $A$ of $[p]\times [p]$ with $a\in A$,
  and Charlie receives $r\in [p]\times [p]$, distributed according to $\mu_p$.

Assume for a contradiction that there is an $(m_1,m_2)$-bit 
  deterministic 
  communication protocol $(\Pi_1,\Pi_2,\Pi_3)$ for $f_{n,d,p}$ over $\mu_{n,d,p}$
  with $m_1=o(nd\log p)$, $m_2<\sqrt{p}$, and failure probability at most $1/40$.
We first use it to obtain an $(m_1,m_2)$-bit 
  public-coin communication protocol
  $\Pi'$ for the single cell communication problem $f_p$ over $\mu_p$
  that has the same failure probability as $\Pi$.
While the first message $\Pi_1'$ of $\Pi'$ has length $m_1$ (which is too large to
  reach a contradiction with Lemma \ref{lem:single}), we show 
  the amount of (internal) information it contains is small:
$$
I( \Pi_1';a\mid A,R)=o(\log p),
$$
where $R$ denotes the public random string of $\Pi'$. 
Then we apply a compression on $\Pi_1'$ to obtain a new 
  public-coin communication protocol $\Pi ''$ for the single-cell problem such that 
  the first message has $o(\log p)$ bits only (while the second message remains the same as $\Pi'$).
This leads to a contradiction with our lower bound
  for the single-cell problem in Lemma \ref{lem:single}.

Consider the single-cell communication problem $f_p$ over $\mu_p$, 
  with $(a,A,r)\sim \mu_p$ as inputs of the three players, respectively.  
Given $\Pi=(\Pi_1,\Pi_2,\Pi_3)$ for $f_{n,d,p}$ over $\mu_{n,d,p}$,
  we describe a public-coin protocol for the single-cell communication problem as follows:
\begin{flushleft}\begin{enumerate}
    \item Public-coin randomness $R$:
    The protocol starts by sampling a public-coin random string $R$. First $R$ contains $i^*\in [n]$ and $j^*\in [d]$
    sampled independently and uniformly at random. For each $(i,j)\in [n]\times [d] $ with $(i,j) \prec (i^*,j^*)$,
    $R$ contains $a_{i,j}$ drawn from $[p]\times [p]$ independently and uniformly at random.
    For each $(i,j)\in [n]\times [d]$ with $(i,j)\succ (i^*,j^*)$, $R$ contains a size-$t$ subset $A_{i,j}$ sampled from $[p]\times [p]$ independently and uniformly at random.
\item Alice, Bob and Charlie run $\Pi$ on $(X,Y,Z)$ constructed as follows and Charlie returns the same bit that $\Pi$ returns.
The input $X$ of Alice contains $(a_{i,j}:i\in [n],j\in [d])$, where 
  $a_{i,j}$ for each $(i,j)\prec (i^*,j^*)$ is from $R$,  
  $a_{i^*,j^*}$ is set to be $a$ (which is her original input in the single-cell problem), and $a_{i,j}$ for each  $(i,j)\succ (i^*,j^*)$ is sampled  from $A_{i,j}$ (from $R$)
independently and uniformly at random using her private randomness.
The input $Y$ of Bob contains $i^*$ and $(A_{i,j}:i\in [n],j\in[d])$, 
  where $A_{i,j}$ for each $(i,j)\prec (i^*,j^*)$ is sampled as a 
    size-$t$ subset of $[p]\times [p]$ that contains $a_{i,j}$ (from $R$)
    independently and uniformly at random using his private randomness,
    $A_{i^*,j^*}$ is set to be $A$ (which is his original input),
    and $A_{i,j}$ for each $(i,j)\succ (i^*,j^*)$ is from $R$.
Finally the input $Z$ of Charlie contains $i^*,j^*$ and $r $
  (which is his original input).
\end{enumerate}\end{flushleft}

It is easy to verify that when $(a,A,r )\sim\mu_p$,
  the triple $(X,Y,Z)$ above is distributed exactly the same as $\mu_{n,d,p}$ and thus,
  the error probability of the above single-cell protocol $\Pi'$ is 
  exactly the same as the error probability of $\Pi$ (at most $1/40$).

We next bound the amount of internal information the first message $\Pi_1'=\Pi_1(X)$ contains:
$$
I(\Pi_1';a \mid A,R)= I(\Pi_1;a\mid A,R).
$$
To this end, note that these four random variables we care about 
  can also be drawn equivalently as follows:
First we draw $(a_{i,j}:i\in [n],j\in [d])$ uniformly and then
  draw $(A_{i,j}:i\in [n],j\in [d])$ as size-$t$ subsets that contain $a_{i,j}$
  uniformly at random.
Next we draw $i^*\in [n]$ and $j^*\in [d]$ uniformly.
Finally we set $\Pi_1$ to be $\Pi_1(a_{i,j}:i\in [n],j\in [d])$, $a$ to be $a_{i^*,j^*}$, $A$ to be $A_{i^*,j^*}$,
  $R$ to be $(i^*,j^*,a^{\prec}_{i^*,j^*},A^{\succ}_{i^*,j^*})$, where we write
  $a^{\prec}_{i^*,j^*}$ to denote the ordered tuple of $a_{i,j}$ with
  $(i,j)\prec (i^*,j^*)$ and $A^{\succ}_{i^*,j^*}$ to denote the ordered tuple
  of $A_{i,j}$ with $(i,j)\succ (i^*,j^*)$.

In this view we have the following sequence of inequalities:
\begin{align*}
    I(\Pi_1; a \mid A, R) 
    = &~ \frac{1}{nd} \sum_{i\in [n],j\in [d]} I(\Pi_{1}; a_{i,j} \mid a^{\prec}_{i,j}, A_{i,j}, A^{\succ}_{i,j})\\
    \leq &~  \frac{1}{nd} \sum_{i\in [n],j\in [d]} I(\Pi_{1}; a_{i,j} \mid  a^{\prec}_{i,j} , A_{1,1}, \ldots, A_{n,d})\\
    =&~ \frac{1}{nd}\cdot  I(\Pi_{1}; a_{1,1}, \ldots, a_{n,d}\mid A_{1,1}, \ldots, A_{n,d}) 
    \le \frac{m_1}{nd} 
    = o(\log p),
\end{align*}
where the second step follows from Fact \ref{fact:mutual-condition} as $A^{\prec}_{i,j}$ is independent of $a_{i,j}$ 
  conditioning on $a^\prec_{i,j}, A_{i.j},$ $A^{\succ}_{i,j}$, the third step follows from the chain rule of mutual information. 

Now we apply Lemma \ref{lem:compression} on the first message $\Pi_1'$ of
  $\Pi$. 
Formally in Lemma \ref{lem:compression} we take $X$ to be the input $a$ of 
  Alice, $Y$ to be the input $A$ of Bob, $R$ to be the public randomness,
  and $M$ to be the first message $\Pi_1'$ from Alice to Bob.
Setting $\eps' =  {1}/{600}$ in Lemma \ref{lem:compression} and noting that 
  $\Pi_1'$ is independent of $Y$ conditioning on $X$ and $R$, 
  we can replace the first message $\Pi_1'$ by a 
  public-coin one-way protocol from Alice to Bob, where Bob uses the 
  random variable $M'$ in Lemma \ref{lem:compression} as the message from
  Alice to continue to run $\Pi'$.
Let $\Pi''$ be the new public-coin protocol.
Then the length of its first message by Lemma \ref{lem:compression} is at most
$$
\frac{o(\log p)+5}{\eps'}+O\left(\log \frac{1}{\eps'}\right)=o(\log p)
$$
and the length of its second message is still at most $\sqrt{p}$.
The error probability of $\Pi''$ is at most 
$$\frac{1}{40} +6\eps'=\frac{1}{40} +\frac{1}{100} < \frac{9}{200} $$
by  Lemma \ref{lem:compression}.
This contradicts with Lemma \ref{lem:single}. Hence we conclude the proof.
\end{proof}

\section{Memory efficient multi-pass algorithm}
\label{sec:algo}

When the continual learner is allowed to take multiple passes over the sequence of tasks, the memory requirement can be significantly reduced. Formally, we recall the following theorem:

\multiPassAlgo*

\paragraph{Notation} Let $\delta = 0.01$ be the confidence parameter and $\gamma = \frac{1}{10c^2}$.
We assume a lexical order for elements in $\X$ and slightly abuse of notation, for any $x_1, x_2 \in \X$ and $q_1, q_2 \in \R$, we write $(x_1, q_1) \preceq (x_2, q_2)$ if $(q_1 < q_2) \vee (q_1 = q_2 \wedge x_1 < x_2)$.
Let $\mu_{\D_i}(x)$ be the probability density of $x$ on the $i$-th distribution $\D_i$ ($i \in [k]$). 
It is WLOG to assume $\mu_{\D_i}(x) = o(\frac{\eps}{c^3})$ for every $i \in [k]$ and $x\in \X$,
as one can always append $\log_2(c/\eps)$ random bits to the data point $x$.

\begin{algorithm}[!t]
\caption{Multi-pass continual learner}
\label{algo:upper}
\begin{algorithmic}[1]
\State Initialize $\alpha = \frac{1}{4}(\frac{2k}{\eps})^{-2/c}$, $\eta = \log\frac{1 - \alpha}{\alpha}$, $\eps_{t} =  (1+\frac{1}{c})^{t} \cdot \frac{\eps}{20c}$, $\forall t\in [0:c]$  
\State Initialize $N = \Theta(\frac{d + \log(c/\delta)}{\alpha})$, $M_1 = \Theta(\frac{c^4\log(kc/\delta)}{\eps_0})$, $M_2 = \Theta(\frac{\log(kc/\delta)}{\eps_0\alpha^2})$
\For{$t = 1, 2, \ldots, c$}
\State $\hat{w}_t \leftarrow 0, S_{t} \leftarrow \D_{1}^{N}$
\For{$i = 1,2, \ldots, k$}
\State $\{(x_{i, t, \tau}, \hat{q}_{i, t, \tau})\}_{\tau \in [t-1]} \leftarrow \textsc{EstimateQuantile}(\D_i, h_1, \ldots, h_{t-1})$
\State $\hat{w}_{i, t} \leftarrow \textsc{EstimateWeight}(\D_i, h_1, \ldots, h_{t-1})$ 
\State $\hat{w}_{t} \leftarrow \hat{w}_t + \hat{w}_{i, t}$
\State $\hat{\D}_{i, t} \leftarrow \textsc{TruncatedRejectionSampling}(\D_i, h_1, \ldots, h_{t-1})$
\ForEach{training data in $S_{t}$}
\State With probability $\hat{w}_{i, t}/\hat{w}_{t}$, replace it with a sample from $\hat{\D}_{i, t}$ 
\EndFor
\EndFor
\State $h_t = \arg\min_{h \in \mH}\ell_{S_t}(h)$
\EndFor
\State Return $h = \mathsf{maj}(\sum_{t\in [c]} h_t)$
\end{algorithmic}
\end{algorithm}

\begin{algorithm}[!htbp]
\caption*{\textsc{EstimateQuantile}($\D_i, h_1, \ldots, h_{t-1}$)}
\begin{algorithmic}[1]
\For{$\tau = 1, 2, \ldots, t-1$}
\State $S^{q}_{i, t, \tau} \sim \D_{i}^{M_1}$
\State $\bar{S}_{i, t, \tau}^{q} =  \{(x, y)\in S_{i,t, \tau}^{q}: (x,\exp(\eta\sum_{\xi=1}^{\nu}\mathbf{1}[h_\xi(x) \neq y])) \preceq (x_{i, t, \nu}, \hat{q}_{i, t,\nu}), \forall \nu \in [\tau-1]\}$
\State Sort $\{(x, \exp(\eta\sum_{\nu=1}^{\tau}\mathbf{1}[h_{\nu}(x) \neq y]))\}_{x\in \bar{S}_{i, t, \tau}^{q}}$ and select the the top $\eps_{t}$-quantile $(x_{i, t, \tau}, \hat{q}_{i, t, \tau})$
\EndFor
\State Return $\{(x_{i, t, \tau}, \hat{q}_{i, t, \tau})\}_{\tau \in [t-1]}$
\end{algorithmic}
\end{algorithm}

\begin{algorithm}[!htbp]
\caption*{\textsc{EstimateWeight}($\D_i, h_1, \ldots, h_{t-1}$)}
\begin{algorithmic}[1]
\State $S^{w}_{i, t} \sim \D_{i}^{M_2}$
\State $\bar{S}_{i, t}^{w} =  \{(x, y)\in S_{i,t}^{w}: (x,\exp(\eta\sum_{\nu=1}^{\tau}\mathbf{1}[h_\nu(x) \neq y])) \preceq (x_{i, \tau}, \hat{q}_{i, t, \tau}), \,\forall \tau \in [t-2]  \}$
\State Return $\frac{1}{|M_2|}\sum_{(x, y)\in \bar{S}_{i, t}^{w}} \min\{\exp(\eta \sum_{\tau=1}^{t-1}\mathbf{1}[h_\tau(x) \neq y]), \hat{q}_{i, t, t-1}\}$
\end{algorithmic}
\end{algorithm}

\begin{algorithm}[!htbp]
\caption*{\textsc{TruncatedRejectionSampling}($\D_i, h_1, \ldots, h_{t-1}$)}
\begin{algorithmic}[1]
\State Sample $(x, y) \sim \D_i$
\State Reject and go back to Step 1 if $\exists \tau \in [t-2]$, s.t. $(x, \exp(\eta\sum_{\nu=1}^{\tau}\mathbf{1}[h_{\nu}(x)\neq y])) \preceq (x_{i, t, \tau}, \hat{q}_{i, t, \tau})$
\State Sample $\lambda$ from $[0, \hat{q}_{i, t, t-1}]$ uniformly 
\State Reject and return to Step 1 if $\lambda \geq \min\{\exp(\eta\sum_{\nu=1}^{t-1}\mathbf{1}[h_{\nu}(x)\neq y]), \hat{q}_{i, t, t-1}\}$
\State Return $(x, y)$
\end{algorithmic}
\end{algorithm}

\paragraph{Overview of the algorithm}
We provide a high level overview of our approach, the pseudocode can be found at Algorithm \ref{algo:upper}. 
Our algorithm is based on the idea of boosting (note that this means it will be improper). 
During each pass $t\in [c]$, the algorithm maintains a {\em hierarchical  truncated} set of the $i$-th data distribution $\D_i$ for each $i\in [k]$. Let 
\begin{align}
    \X_{i, t, \tau} = \left\{x \in \X: (x, \exp(\sum_{\xi=1}^{\nu} \mathbf{1}[h_{\xi}(x) \neq y])) \preceq (x_{i, t, \nu}, \hat{q}_{i, t, \nu}), \forall \nu \in [\tau-1]\right\}, \quad \forall \tau \in [t] \label{eq:xit}\end{align}
and we take $\X_{i, t, 1} = \cdots = \X_{k,t, 1} = \X$ in the above definition. These are hierarchical sets and with high probability satisfy (See Lemma \ref{lem:hier} for the proof)
\begin{enumerate}
    \item[(1)] $\X_{i,t, t} \subseteq \X_{i, t, t-1} \subseteq \cdots \subseteq \X_{i, t, 1}, \quad \forall i \in [k], t\in [c]$;
    \item[(2)] $\X_{i, c, \tau} \subseteq \cdots \subseteq \X_{i, \tau+1, \tau} \subseteq \X_{i,\tau, \tau}, \quad \forall i \in [k], \tau \in [c]$. 
\end{enumerate}
These sets are obtained via the {\sc EstimateQuantile} procedure, which iteratively computes the top $\eps_t$-quantile of $\D_i$.

Roughly speaking, $\mathcal{X}_{i, t, t-1}$ is the ``support'' of $\D_{i}$ in the $t$-th pass.
Our algorithm performs multiplicative weights update over every data point in $\X$ and the distribution $\D_i$ gets reweighted at every round. 
{\sc EstimateWeight} estimates the updated weight of $\X_{i, t,t-1}$ with the top $\eps_t$-quantile truncated.
{\sc TruncatedRejectionSampling} has access to the original distribution $\D_i$ and performs rejection sampling to get samples from the truncated distribution.

We sample the training set $S_{t}$ of the $t$-th round in a streaming manner. 
Initially we add $N$ dummy samples. In the $i$-th task, we discard old data and replace them with a new one with probability $\hat{w}_{i, t}/\hat{w}_t$.
We run ERM over the training set $S_{t}$ at the end, and derive a hypothesis $h_t$.
The final output is determined by the majority vote $h = \mathsf{maj}(\sum_{t\in [c]}h_t)$, defined as 
\[
h(x) = \mathsf{maj}\left(\sum_{t\in [c]} h_t(x)\right) = \left\{
\begin{matrix}
1 & \sum_{t\in [c]} h_t(x) \geq c/2\\
0 & \text{otherwise}
\end{matrix}
\right.
\]

We sketch the proof of Theorem \ref{thm:upper} and detailed proof of this section can be found in the Appendix \ref{sec:algo-app}.
We first prove that the \textsc{EstimateQuantile} procedure provides good estimations on the quantile each step. 
\begin{lemma}[Quantile estimation]
\label{lem:estimate-quantile}
With probability at least $1 - \frac{\delta}{20}$, the following event holds
\begin{enumerate}
\item For any $i \in [k], t\in [c], \tau \in [t-1]$,
\begin{align*}
    \Pr_{(x, y)\sim \D_{i}| x\in \X_{i, t, \tau}}\left[\left(x, \exp(\eta\sum_{\nu=1}^{\tau} \mathbf{1}\{h_{\tau}(x)\neq y\})\right) \succeq (x_{i, t, \tau}, \hat{q}_{i, t, \tau})\right] \in \left[(1- \gamma)\eps_t, (1+\gamma)\eps_t\right].
\end{align*}
Here $\D_{i}| x\in \X_{i, t, \tau}$ denotes the conditional distribution of $\D_i$ on the event of $x\in \X_{i,t, \tau}$.
\item Denote $P_{i, t, \tau} :=  \Pr_{(x, y)\sim \D_i}[x\in \mathcal{X}_{i, t, \tau}]$. For any $i \in [k], t\in [c], \tau \in [t-1]$,
\begin{align*}
   P_{i, t, \tau+1} \in \left[(1-(1+\gamma)\eps_t)P_{i, t, \tau}, (1-(1-\gamma)\eps_t)P_{i, t, \tau}\right].
\end{align*}
\end{enumerate}
\end{lemma}

We next prove the sets $\{X_{i, t, \tau}\}_{i\in [k], t\in [c], \tau \in [t]}$ have a hierarchical structure.
\begin{lemma}[Hierarchical set]
\label{lem:hier}
Condition on the event of Lemma \ref{lem:estimate-quantile}, we have
\begin{enumerate}
    \item $\X_{i,t, t} \subseteq \X_{i, t, t-1} \subseteq \cdots \subseteq \X_{i, t, 1}, \quad \forall i \in [k], t\in [c]$;
    \item $\X_{i, c, \tau} \subseteq \cdots \subseteq \X_{i, \tau+1, \tau} \subseteq \X_{i,\tau, \tau}, \quad \forall i \in [k], \tau \in [c]$. 
\end{enumerate}
\end{lemma}

We formally define the updated distribution at each pass.
Ideally, we sample from the {\em truncated} distribution.
\begin{definition}[Truncated distribution]
\label{def:trun}
For any $i \in [k], t\in [c]$, define the truncated distribution $\D_{i, t, \trun}$ as
\begin{align*}
    \mu_{\D_{i, t,\trun}}(x) = \frac{\mathbf{1}[x \in \X_{i, t, t-1}]\cdot \mu_{\D_i}(x) \cdot \min\{\exp(\eta \sum_{\tau=1}^{t-1}\mathbf{1}[h_\tau(x) \neq y]), \hat{q}_{i, t, t-1}\}}{\sum_{x'\in \X}\mathbf{1}[x' \in \X_{i, t, t-1}]\cdot \mu_{\D_i}(x') \cdot \min\{\exp(\eta \sum_{\tau=1}^{t-1}\mathbf{1}[h_\tau(x') \neq y]), \hat{q}_{i, t, t-1}\} }, \quad \forall x\in \X.
\end{align*}
The mixed truncated distribution $\D_{t, \trun}$ is defined as
\begin{align*}
    \mu_{\D_{t, \trun}}(x) = \frac{\sum_{i=1}^{k}\mathbf{1}[x \in \X_{i, t, t-1}]\cdot \mu_{\D_i}(x) \cdot \min\{\exp(\eta \sum_{\tau=1}^{t-1}\mathbf{1}[h_\tau(x) \neq y]), \hat{q}_{i, t, t-1}\}}{\sum_{i=1}^{k}\sum_{x'\in \X}\mathbf{1}[x' \in \X_{i, t, t-1}]\cdot \mu_{\D_i}(x') \cdot \min\{\exp(\eta \sum_{\tau=1}^{t-1}\mathbf{1}[h_\tau(x') \neq y]), \hat{q}_{i, t, t-1}\} }, \quad \forall x\in \X.
\end{align*}
To connect $\D_{t, \trun}$ and $\D_{i,t, \trun}$, we define the weights $\{w_{i, t}\}_{i\in [k]}$
\begin{align*}
w_{i, t} =\sum_{x\in \X}\mathbf{1}[x \in \X_{i, t, t-1}]\cdot \mu_{\D_i}(x) \cdot \min\left\{\exp(\eta \sum_{\tau=1}^{t-1}\mathbf{1}[h_\tau(x) \neq y]), \hat{q}_{i, t, t-1}\right\}
\end{align*}
and $p_{i, t} = \frac{w_{i, t}}{\sum_{i'\in [k]}w_{i', t}}$. 
The mixed distribution $\D_{i, \trun}$ can be derived by first drawing an index $i$ from $[k]$ according to $\{p_{i, t}\}_{i\in [k]}$, and then drawing from $\D_{i, t, \trun}$. That is to say, $\D_{t, \trun} = \sum_{i=1}^{k}p_{i, t}\D_{i, t, \trun}$.
\end{definition}

We next prove that the {\sc EstimateWeight} procedure approximately estimates the weight of each distribution $\D_{i, t}$.
\begin{lemma}[Weight estimation]
\label{lem:estimate-weight}
Condition on the event of Lemma \ref{lem:estimate-quantile}, with probability at least $1 - \frac{\delta}{20}$, the {\sc EstimateWeight} procedure outputs $\hat{w}_{i, t} = (1\pm \alpha/8)w_{i, t}$. 
Furthermore, define $\hat{p}_{i, t} = \frac{\hat{w}_{i, t}}{\sum_{i'\in [k]}\hat{w}_{i', t}}$, then $\hat{p}_{i, t} = (1\pm \alpha/3)p_{i, t}$ for any $i \in [k], t\in [c]$.
\end{lemma}

We then analyse the {\sc TruncatedRejectionSampling} process and prove that it returns a sample from $\D_{i, t, \trun}$ without much overhead on sample complexity.
\begin{lemma}[Truncated rejection sampling]
\label{lem:truncated-rejection-sampling}
For any $i\in [k], t\in [c]$, condition on the event of Lemma \ref{lem:estimate-quantile}, the {\sc TruncatedRejectionSampling}($\D_i, h_1, \ldots, h_{t-1}$) returns a sample from the distribution of $\D_{i,t, \trun}$. 
With probability at least $1 - \frac{\delta}{20}$, the total number of samples that have been drawn for the training set $\{S_{t}\}_{t\in [c]}$ is at most
$
O\left(\frac{dc^2\log^3(kdc/\eps\delta)}{\eps\alpha}\right)
$.
\end{lemma}

For any $t\in [c]$, define $\hat{\D}_{t, \trun} = \sum_{i=1}^{t}\hat{p}_{i, t}\D_{i, t, \trun}$. Recall that our algorithm samples and updates the training data $S_t$ in a streaming manner. It is easy to verify that each data of $S_t$ is drawn from $\hat{\D}_{t, \trun}$.

\begin{lemma}
\label{lem:training-set}
For any $t\in [c]$, the training set $S_t$ is drawn from $\hat{\D}_{t, \trun}^{N}$.
\end{lemma}

Additionally, combining Lemma \ref{lem:estimate-weight} and Lemma \ref{lem:truncated-rejection-sampling}, one can prove $\hat{\D}_{t, \trun}$ is close to $\D_{t, \trun}$ 
\begin{lemma}[Bounded TV distance]
\label{lem:distribution-tv}
For any $t\in [c]$, conditioned on the events of Lemma \ref{lem:estimate-quantile} and Lemma \ref{lem:estimate-weight}, the empirical truncated distribution $\hat{\D}_{t, \trun}$ and truncated distribution $\D_{t, \trun}$ satisfies $\|\hat{\D}_{t, \trun} - \D_{t, \trun}\|_{\TV} \leq \alpha/3$.
\end{lemma}

By a standard application of VC theory, we can show the hypothesis $h_t$ returned by Algorithm \ref{algo:upper} in the $t$-th pass achieves good accuracy on $\D_{t, \trun}$.

\begin{lemma}
\label{lem:acc}
For any $t\in [c]$, condition on the events of Lemma \ref{lem:distribution-tv}, with probability at least $1 - \frac{\delta}{20c}$, we have $\ell_{\D_{t, \trun}}(h_t) \leq \frac{\alpha}{2}$.
\end{lemma}

Now we are able to derive the accuracy guarantee.
\begin{lemma}
\label{lem:boosting-correct}
Let $h = \mathsf{maj}(\sum_{t\in [c]} h_t)$ be the function returned by Algorithm \ref{algo:upper}, with probability at least $1 - \frac{\delta}{4}$, $\ell_{\D_i}(h) \leq \eps$ holds for any $i \in [k]$.
\end{lemma}
\begin{proof}
We proceed the proof by conditioning on the event of Lemma \ref{lem:estimate-quantile}, Lemma \ref{lem:estimate-weight} and Lemma \ref{lem:acc}.
The key is to track the following potential
\begin{align*}
    \Phi_{t} = \frac{1}{k}\sum_{i=1}^{k} \sum_{x \in \X_{i, t,t-1}}\mu_{\D_i}(x)\exp(\eta\sum_{\tau=1}^{t-1}\mathbf{1}[h_{\tau}(x) \neq y]), \quad t \in [c+1].
\end{align*}
Note that, when $t = 1$, we take $\X_{i, t, 0} = \X$ and therefore $\Phi_1 = \frac{1}{k}\sum_{i=1}^{k} \sum_{x \in \X}\mu_{\D_i}(x) = 1$;
When $t=c+1$, for the purpose of purpose, we extend the definition to $\X_{i, c+1, c}$ by imagining an extra pass and assuming that the hierarchical property holds (i.e., Lemma \ref{lem:hier}) as well.

For any $t\in [2:c]$, we have that 
\begin{align}
    &~\log \Phi_{t+1} - \log \Phi_{t}\notag\\
    = &~  \log \sum_{i=1}^{k}\sum_{x\in \X_{i, t+1,t}} \frac{\mu_{\D_i}(x)\exp(\eta\sum_{\tau=1}^{t}\mathbf{1}[h_{\tau}(x) \neq y])}{\sum_{i=1}^{k}\sum_{x'\in \X_{i, t,t-1}}\mu_{\D_i}(x')\exp(\eta\sum_{\tau=1}^{t-1}\mathbf{1}[h_{\tau}(x') \neq y])}\notag\\
    \leq &~   \log \sum_{i=1}^{k}\sum_{x\in \X_{i, t,t}} \frac{\mu_{\D_i}(x)\exp(\eta\sum_{\tau=1}^{t}\mathbf{1}[h_{\tau}(x) \neq y])}{\sum_{i=1}^{k}\sum_{x'\in \X_{i, t,t-1}}\mu_{\D_i}(x')\exp(\eta\sum_{\tau=1}^{t-1}\mathbf{1}[h_{\tau}(x') \neq y])}\notag\\
    = &~ \log \sum_{i=1}^{k}\sum_{x\in \X_{i, t, t}} \frac{\mu_{\D_i}(x)\exp(\eta\sum_{\tau=1}^{t-1}\mathbf{1}[h_{\tau}(x) \neq y])}{\sum_{i=1}^{k}\sum_{x'\in \X_{i, t,t-1}}\mu_{\D_i}(x')\exp(\eta\sum_{\tau=1}^{t-1}\mathbf{1}[h_{\tau}(x') \neq y])} \cdot \exp(\eta \mathbf{1}[h_{t}(x) \neq y]))\notag\\
    \leq &~ \log \sum_{i=1}^{k}\sum_{x\in \X_{i, t,t}} \frac{\mu_{\D_i}(x) \cdot \min \{\exp(\eta\sum_{\tau=1}^{t-1}\mathbf{1}[h_{\tau}(x) \neq y]) , \hat{q}_{i, t, t-1}\}}{\sum_{i=1}^{k}\sum_{x'\in \X_{i, t,t-1}}\mu_{\D_i}(x')\cdot \min\{\exp(\eta\sum_{\tau=1}^{t-1}\mathbf{1}[h_{\tau}(x') \neq y]), \hat{q}_{i, t, t-1} \}} \cdot \exp(\eta \mathbf{1}[h_{t}(x) \neq y]))\notag\\
    \leq &~ \log \sum_{i=1}^{k}\sum_{x\in \X_{i, t,t-1}} \frac{\mu_{\D_i}(x) \cdot \min \{\exp(\eta\sum_{\tau=1}^{t-1}\mathbf{1}[h_{\tau}(x) \neq y]) , \hat{q}_{i, t,t-1}\}}{\sum_{i=1}^{k}\sum_{x'\in \X_{i, t,t-1}}\mu_{\D_i}(x')\cdot \min\{\exp(\eta\sum_{\tau=1}^{t-1}\mathbf{1}[h_{\tau}(x') \neq y]), \hat{q}_{i, t,t-1} \}} \cdot \exp(\eta \mathbf{1}[h_{t}(x) \neq y]))\notag\\
    = &~ \log ( 1 - \ell_{\D_{t, \trun}}(h_{t}) + \ell_{\D_{t, \trun}}(h_t) \cdot \exp(\eta))\notag \\
    \leq &~  \log (1 - \alpha +  \alpha\cdot  \exp(\eta) )\notag\\
    \leq &~  \log 2(1-\alpha) \label{eq:upper1}
\end{align}
The second step follows from $\X_{i, t+1, t} \subseteq \X_{i,t,t}$ (see Lemma \ref{lem:hier}), the fourth step follows from
\[
\exp(\eta\sum_{\tau=1}^{t-1}\mathbf{1}[h_{\tau}(x) \neq y]) \leq \hat{q}_{i, t,t-1}, \quad \forall  x\in \X_{i, t,t}.
\]
The fifth step follows from $\X_{i, t,t} \subseteq \X_{i, t, t-1}$, the sixth step follows from the definition of $\ell_{\D_{t, \trun}}(h_t)$ and the seventh step follows from $\ell_{\D_{t, \trun}}(h_t) \leq \alpha$ (Lemma \ref{lem:acc}) and we plug in $\eta = \log(\frac{1-\alpha}{\alpha})$ in the last step.

For any $i \in [k]$, define 
\[
\X_{i, \bad} = \left\{x \in \X_{i, c+1,c}: \sum_{t=1}^{c}\mathbf{1}[h_t(x) \neq y] \geq c/2\right\}.
\]
That is to say, $\X_{i, \bad}$ contains all data points that are wrongly labeled by function $h$ in $\X_{i, c+1,c}$. 
Note that
\begin{align}
     \Phi_{c+1} = &~ \frac{1}{k}\sum_{i'=1}^{k} \sum_{x \in \X_{i', c+1, c}}\mu_{\D_{i'}}(x)\exp(\eta\sum_{t=1}^{c}\mathbf{1}[h_{t}(x) \neq y]) \notag \\
     \geq &~ \frac{1}{k}\sum_{x \in \X_{i, c+1,c}}\mu_{\D_{i}}(x)\exp(\eta\sum_{t=1}^{c}\mathbf{1}[h_{t}(x) \neq y]) \notag \\
     \geq &~ \frac{1}{k}\sum_{x \in \X_{i,\bad}}\mu_{D_i}(x) \exp(\eta c/2). \label{eq:acc2}
\end{align}
Now taking a telescopic summation on Eq.~\eqref{eq:upper1}, we have 
\begin{align*}
   c\log2(1-\alpha) \geq &~ \log \Phi_{c+1} - \log \Phi_{1} = \log \Phi_{c+1} \geq \eta c /2 + \log \sum_{x \in \X_{i,\bad}}\mu_{D_i}(x) - \log k\\
   = &~ \frac{c}{2}\log\frac{1-\alpha}{\alpha} + \log \sum_{x \in \X_{i,\bad}}\mu_{D_i}(x) - \log k ,
\end{align*}
where the second step holds as $\Phi_1 = 1$, the third step holds due to Eq.~\eqref{eq:acc2}, we plug in $\eta = \log \frac{1-\alpha}{\alpha}$ in the last step.
Rearranging these terms and plug in $\alpha = \frac{1}{4} (\frac{2k}{\eps})^{-2/c}$,
\begin{align*}
    \log \sum_{x \in \X_{i,\bad}}\mu_{\D_i}(x) - \log k \leq c\log 2\sqrt{\alpha(1-\alpha)} \leq \log(\eps/2k),
\end{align*}
Therefore, we have
\begin{align}
\label{eq:acc4}
\sum_{x \in \X_{i,\bad}}\mu_{\D_i}(x) \leq \frac{\eps}{2}.
\end{align}
For those outside of $\X_{i, \bad}$, we have
\begin{align}
\label{eq:acc5}
\Pr_{x\sim \D_i}[x \notin \X_{i, c+1,c}] \leq  1 - (1-(1-\gamma)\eps_{c+1})^{c} \leq 3c\eps_0  \leq \frac{\eps}{2}.
\end{align}
The first step follows from Lemma \ref{lem:estimate-quantile}, the third step holds as $\eps_0 = \frac{\eps}{10c}$.
Combining Eq.~\eqref{eq:acc4} and Eq.~\eqref{eq:acc5}, we conclude $\ell_{\D_i}(h) \leq \eps$ and finish the proof.
\end{proof}

We now wrap up the proof of Theorem \ref{thm:upper}.
\begin{proof}[Proof of Theorem \ref{thm:upper}]
The accuracy guarantee of Algorithm \ref{algo:upper} has already been established by Lemma \ref{lem:boosting-correct}. We next calculate the memory requirement and sample complexity.

\paragraph{Memory requirement}
For the memory requirement, we need (1) storing the training data $S_{t}$, which takes $O(N) \leq  O((\frac{k}{\eps})^{2/c}db \log(kc/\delta))  )$ bits;
(2) storing samples $S_{i, t}^{q}$ for {\sc EstimateQuantile}, which takes $O(bc + M_1 b \eps_0) = O(bc^4\log(kc/\delta))$ bits using streaming implementation;
(3) storing samples $S_{i, t}^{w}$ for {\sc EstimateWeight}, which takes $O(\log(k/\eps))$ bits using streaming implementation;
(4) storing the hypothesis $h_1, \ldots, h_{c}$, which takes $O(dbc)$ bits; 
Hence the total space requirement is
\[
O\left( \left(\frac{k}{\eps}\right)^{2/c}db\log(kc/\delta) + bc^4\log(kc/\delta) + dbc + \log(k/\eps) \right).
\]

\paragraph{Sample complexity} Algorithm \ref{algo:upper} draws $(M_1c + M_2)kc = O(kc^7\log(kc/\delta)\eps^{-1} + (\frac{k}{\eps})^{4/c}kc^2\log(kc/\delta)\eps^{-1})$ samples for estimating weights and quantile. By Lemma \ref{lem:truncated-rejection-sampling}, with probability at least $1-\frac{\delta}{20}$, the training set $\{S_{t}\}_{t\in [c]}$ takes $O\left(\frac{dc^2\log^3(kdc/\eps\delta)}{\eps\alpha}\right)$
 samples. Hence the sample complexity is 
\[
O\left(\left(\frac{k}{\eps}\right)^{2/c}dc^2\log^3(kdc/\eps\delta)\eps^{-1} +   kc^7\log(kc/\delta)\eps^{-1} + \left(\frac{k}{\eps}\right)^{4/c}kc^2\log(kc/\delta)\eps^{-1}\right).
\]
This concludes the proof.
\end{proof}


\section{An exponential separation result for multi-pass learning}
\label{sec:lower-multi}

We prove a lower bound for the memory 
  requirement of $c$-pass {\em proper} continual learners, showing that making $c$ passes can only help save memory by  a factor of no more than $c^3$.
This is in sharp contrast to the $1/c$ saving in the exponent by the improper learner presented in the last section.

\multiPassLower*

\subsection{Pointer chasing}
We reduce from the classic communication problem of pointer chasing.

\begin{definition}[Pointer chasing \cite{papadimitriou1982communication}]\label{def:pointer} 
Let $n$ and $c$ be two positive integers.
In a pointer chasing problem, Alice gets an input map $f_A: [n]\rightarrow [n]$ and Bob receives another input map $f_B: [n]\rightarrow [n]$. The pointers $w^{(1)}, w^{(2)}, \ldots $ are recursive defined as
\[
w^{(1)} = 1,  ~~ w^{(2)} = f_{A}(w^{(1)}), ~~ w^{(3)} = f_{B}(w^{(2)}), ~~ w^{(4)} = f_A(w^{(3)}), ~~ w^{(5)} = f_B(w^{(4)}),~~ \ldots
\]
In the communication problem, Alice speaks first and the communication proceeds in at most $2c - 1$ rounds. The goal is for Bob to output  the binary value of $w^{(2c+2)} \pmod 2$.
\end{definition}

The problem requires $\Omega(n/c - c\log n)$ bits of (total) communication: 
\begin{theorem}[\cite{nisan1991rounds, klauck2000quantum, yehudayoff2020pointer}]
\label{thm:point-chasing}
The randomized communication complexity of the pointer chasing problem (with  error probability at most $1/3$) 
is at least $n/(2000c) - 2c\log n$.
\end{theorem}

\subsection{Hypothesis class}

Let $c,k,d$ and $b$ be four positive integers. 
We describe the hypothesis class $\calH=\calH_{c,k,d,b}$ that will be used to prove our lower bound in Theorem \ref{thm:multi-pass-lower}.
We start with the data universe $\X$ of $\calH$:

\paragraph{Data universe.}
Let $p$ be the smallest prime with $p\ge (kdb)^b$. 
The data universe $\X$ consists of $\X = Y\cup Z$, where $Y$ and $Z$ can be further partitioned 
  into
$$Y = Y_{ 1}\cup \cdots \cup Y_{ k}\quad\text{and}\quad
Z = Z_{ 1}\cup \cdots, Z_{ k},$$
and each of $Y_{ i}$ and $Z_{ i}$ is a disjoint copy of $[2d]\times \mathbb{F}_p$.

For notation convenience, we use a $4$-tuple $x = (x_1, x_2, x_3, x_4 )$ to represent a data point $x \in \X$, where $x_1 \in \{1, 2\}$ determines $x \in Y$ (when $x_1=1$) or $x \in Z$ (when $x_1=2$), 
$x_2 \in [k]$ indicates that $x$ is from $Y_{x_2}$ or $Z_{x_2}$ (depending on the value of $x_1$), $x_3 \in [2d]$  and $x_4 \in \mathbb{F}_p$.

We will refer to a triple $q\in [k]\times [d]\times [b]$ as a \emph{pointer} 
  (looking ahead we will set $n=kdb$ in our reduction from the pointer chasing problem over $[n]$),
  and write $Q=[k]\times [d]\times [b]$ to denote the set of pointers.
Let $B=\{1,\ldots, (kdb)^b \}$ be $(kdb)^b$ elements in 
  $\mathbb{F}_p$; the choices do not matter. 
Given $|B|=(kdb)^b$, we view each element
  $a\in B$ as a tuple $(a_1,\ldots,a_b)$ of $b$ pointers with $a_i\in Q$;
  this can be done by picking an arbitrary bijection between
  $B$ and $Q^b$.

\paragraph{Hypothesis class.}
Our main challenge lies in the construction of the hypothesis class $\mH$.
First we review the Reed-Solomon code, which is a key building block of our construction.

\def\RS{\mathsf{RS}}

\begin{definition}[Reed-Solomon code]
\label{def:rs-code}
Let $p$ be a prime number and let $m, n$ be two positive integers that  satisfy $n \leq m \leq p$. The 
\emph{Reed-Solomon code} $\mathsf{RS}_{n, m}: \F_p^{n}\rightarrow \F_p^{m}$ maps $\F_p^{n}$ to $\F_p^{m}$ by setting
\begin{align*}
(z_1,\ldots,z_m)=\RS_{n,m}(y_1,\ldots,y_n),\quad\text{where
    $z_j = \sum_{i\in [n]}y_i f_j^{i-1}, ~~\forall j \in [m]$,}
\end{align*}
where $f_j$ is the $j$-th element of $\F_p$ under an arbitrary ordering.
\end{definition}

We need the following error correction guarantee of Reed-Solomon code.
\begin{lemma}
\label{lem:err-rs}
For any distinct $y, y' \in \F_p^{n}$,  $z = \mathsf{RS}_{n, m}(y) \in \F_p^{m}$ and $z' = \mathsf{RS}_{n, m}(y') \in \F_p^{m}$ satisfy
\[
\big|j\in [m]: z_j \neq z_j'\big| \geq \left(1 - \frac{n}{m}\right)m.
\]
\end{lemma}
\def\nil{\textsf{nil}}

We can describe the hypothesis class $\mH$ now.
Each function $h_A:\X\rightarrow \{0,1\}$ is specified by
  a tuple $A=(a_{i,j}\in B:i\in [2c+1],j\in [d])$.
As discussed earlier, we view each $a_{i,j}$ in $A$ as a tuple 
  $(a_{i,j,\ell}:\ell\in [b])$ from $Q^b$.
We use the following procedure to determine $h_A$ over $\X=Y \cup Z$:
\begin{flushleft}\begin{enumerate}
\item Initialization: Let $I :[k]\rightarrow [2c+1]\cup \{\nil\}$ and 
  $J:[k]\rightarrow [2c+2]\cup \{\nil\}$ be two arrays such that all
  entries are set to $\nil$.
Let $q^{(1)}\in Q$ be a point set to be $(1,1,1)$.
\item For each $\tau=2,\ldots,2c+2$ (letting $q$ denote the previous pointer $q^{(\tau-1)}$):
 \begin{enumerate}
    \item If $\tau$ is even and $I (q_1)=\nil$, 
      we set $h_A(x)$ for all data points $x$ in $Y_{q_1}$ as follows
      (the fact that $I (q_1)=\nil$ makes sure that we have not set 
      $h_A(x)$ for points in $Y_{1,q_1}$ before).
      Let
      \begin{equation}\label{eq:hehe3}
(z_1,\ldots,z_{2d})=\RS_{d,2d} (a_{\tau-1,1},\ldots,a_{\tau-1,d} )\in \F_p^{2d}.
\end{equation}
      
      For each data point $x\in \X$ with
      $x_1=1$ and $x_2=q_1$,
            \begin{align*}
                h_A(x) =  
                \begin{cases}
                1 & \text{if $x_4 = z_{ x_3}$}\\
                0 & \text{if $x_4 \neq z_{ x_3}$} 
                \end{cases}
            \end{align*}
        Recall that $x_3\in [2d]$ and $x_4\in \F_p$.
        We set $I (q_1)=\tau-1$ (since $h_A$ has been set for points in
        $Y_{ q_1}$ using $(a_{\tau-1,j}:j\in [d])$ from $A$).
        Before moving to the next loop, we set the new pointer
          $q^{(\tau)}$ to be $a_{\tau-1,q_2,q_3}\in Q$ (recall that $q_2\in [d]$,
          $q_3\in [b]$, and
          we view each $a_{\tau-1,q_2}\in B$ as a tuple from $Q^b$).
      \item If $\tau$ is even and $t=I (q_1)\ne \nil$, we don't need to set 
        $h_A$ for points in $Y_{q_1}$ (since they have already been set using
        $(a_{t,j}:j\in [d])$).
     We just set the new pointer $q^{(\tau)}=a_{t,q_2,q_3}$.
            \item 
            The cases with $\tau$ being odd is symmetric, with $Y $ being replaced by $Z$ and $I $ being replaced by $J$.
            If $\tau$ is odd and $J(q_1)=\nil$,
            we set $h_A(x)$ for points in $Z_{ q_1}$ as follows: Let
      \begin{equation}\label{eq:hehe4}
(z_1,\ldots,z_{2d})=\RS_{d,2d} (a_{\tau-1,1},\ldots,a_{\tau-1,d} )\in \F_p^{2d}.
\end{equation}
    For each data point $x\in \X$ with
      $x_1=2$ and $x_2=q_1$, we set
            \begin{align*}
                h_A(x) =  
                \begin{cases}
                1 & \text{if $x_4 = z_{ x_3}$}\\
                0 & \text{if $x_4 \neq z_{ x_3}$} 
                \end{cases}
            \end{align*}
        We set $J(q_1)=\tau-1$ and set the new pointer
          $q^{(\tau)}$ to be $a_{\tau-1,q_2,q_3}$.
 \item If $\tau$ is odd and $t=J(q_1)\ne \nil$, 
     we just set the new pointer $q^{(\tau)}=a_{t,q_2,q_3}$.
     
        
    \end{enumerate}
\item For any $Y_{ i}$ with $I (i)=\nil$ (so these data points have not been set yet), we set $h_A(x)=1$ for all points in $Y_{ i}$; similarly for any $Z_{ i}$ with $J(i)=\nil$, we set $h_A(x)=1$ for all points in $Z_{ i}$.
    \end{enumerate}\end{flushleft}
This finishes the description $h_A$ and the hypothesis class $\mH$.

We first observe that the VC dimension of $\mH$ is bounded above by $O(cd)$.
\begin{lemma}
\label{lem:bounded-VC-multi}
The VC dimension of $\mH$ is at most $4cd + 2c + 2d + 1$.
\end{lemma}

\begin{proof}

For any subset $S \subseteq \X$ with size $|S| =4cd + 2c + 2d + 2$, we prove that the hypothesis class $\mH$ can not shatter $S$.

\vspace{+2mm}
{\noindent \bf Case 1 \ \ } $|S \cap Y| \geq (c+1)(2d+1)+1$.
Let $S_{Y} = \{x^{(1)}, \ldots, x^{((c+1)(2d+1)+1)}\} \subseteq S \cap Y$. First, if $S_Y$ contains data points from at least $c+2$ blocks, that is, there exists $c+2$ indices $i_1, \ldots, i_{c+2}$ such that $Y_{i_1} \cap S_Y \neq \emptyset, \ldots, Y_{i_{c+2}} \cap S_Y \neq \emptyset$, then $\mH$ can not shatter $(0, \ldots, 0)$. The reason is that for any function $h \in \mH$, there are at most $(c+1)$ blocks contain label $0$ data points, i.e., $|\{i\in [k]: I(i) \neq \nil\}| \leq c+1$.
On the other hand, if $S_Y$ contains data points from at most $(c+1)$ blocks, then by pigeonhole principle, there exists an index $i \in [k]$ such that $|S_Y \cap Y_{i}| \geq 2d + 2$. Note that any function $h$ in $\mH$ either labels all data points in $Y_{i}$ with 1, or labels at most $2d$ data points with $1$, hence $\mH$ can not shatter $(1, \ldots, 1, 0)$.

\vspace{+2mm}
{\noindent \bf Case 2 \ \ } $|S \cap Z| \geq c(2d+1)+1$. This follows from a similar argument of the first case.

In summary, whenever $|S| \geq (2c+1)(2d+1) + 1 = 4cd + 2c + 2d + 2$, $\mH$ can not shatter the set $S$. 
Hence the VC dimension of $\mH$ is at most $4cd + 2c + 2d + 1$.
\end{proof}

\subsection{Reduction from pointer chasing to proper learning}

We prove Theorem \ref{thm:multi-pass-lower} by reducing from the communication 
  problem of pointer chasing.
\begin{proof}[Proof of Theorem \ref{thm:multi-pass-lower}]
For notation convenience, we use $\eps,c,k',d'$ and $b'$ to denote parameters in the statement of Theorem \ref{thm:multi-pass-lower} (because we will use 
  $k,d$ and $b$ to denote integers used to specify the hypothesis class $\mH$).
So $\eps\in (0,1/4]$ and $c,k',d',b'$ are positive integers such that 
  $d'\ge C_0 c$ and $b'\ge C_0 \log(k'd'/(c\eps))$ for some sufficiently large constant $C_0$.
  
Let $\alpha=\lfloor 1/(4\eps)\rfloor$ and 
  $k=k'\alpha=\Theta(k'/\eps)$.
Let $d$ be the largest integer such that 
$$
4cd + 2c + 2d + 2 \le d'
$$
so we have $d=\Theta(d'/c)$ given that $d'\ge C_0c$.
Finally, we choose  $b$ to be the largest integer such that the description size
  $\log_2 |\X|$ of hypothesis class $\calH_{c,k,d,b}$ is at most $b'$.
Given that
$$
\log_2 |\X|=O\big(b\log (kdb)\big)
$$
and $b'\ge C_0\log(k'd'/(c\eps))$,
  we have that $b$ is a positive integer that satisfies
$$
b=\Omega\left(\frac{b'}{\log(k'd'b'/( c\eps))}\right).
$$
With these choices of $k,d$ and $b$, we use 
  $\calH=\calH_{c,k,d,b}$ as the hypothesis class in the rest of the proof,
  which has VC dimension at most $d'$ and description size at most $b'$.
We prove below that any $c$-pass $(\eps,0.01)$-continual learner
  for $\mH$ over $2k'$ tasks must require $\widetilde{\Omega}(kdb/c^2)$ memory;
  the theorem follows since by our choices of $k,d$ and $b$, we have
$$
\frac{kdb}{c^2}=\Omega\left(\frac{k'}{\eps}\cdot \frac{d'}{c}\cdot 
  \frac{b'}{\log (k'd'b'/(c\eps))}\cdot \frac{1}{c^2}\right)
=\widetilde{\Omega}\left(\frac{k'd'b'}{c^3\eps}\right)
=\frac{1}{c^2}\cdot \widetilde{\Omega}\left(\frac{k'd'b'}{c\eps}\right).
$$

For this purpose, we show that any $c$-pass $(\eps,0.01)$-continual learner
  for $\mH$ over $2k'$ tasks with an $m$-bit memory can be used to 
  obtain a protocol for the pointer chasing communication problem
  with parameters $n\coloneqq kdb$ and $c$ (as in Definition \ref{def:pointer})
  that has total communication $2cm$ and success probability $0.99$.
Our goal then follows from Theorem \ref{thm:point-chasing}. 


\paragraph{Reduction.} 
With $n\coloneqq kdb$, we let $\rho$ be the following bijection between
  $Q=[k]\times [d]\times [b]$ and $[n]$: 
$$
\rho(i,j,\ell)=(i-1)db+jb+\ell,\quad\text{for all $(i,j,\ell)\in Q.$}  
$$
with $\rho(1,1,1)=1$. 

On input $f_A:[n]\rightarrow [n]$, 
Alice creates $k'$ data distributions $\D_1, \ldots, \D_{k'}$ over
  $\X$ as follows.
First she views $f_A$ as a tuple of pointers $(y_{i,j,\ell}\in Q: i\in [k],
  j\in [d],\ell\in [b])$, with $$y_{i,j,\ell}=\rho^{-1}\Big(f_A\big(\rho(i,j,\ell)\big)\Big)\in Q.$$
She also views $y_{i,j}=(y_{i,j,1},\ldots,y_{i,j,b})\in Q^b$ as an 
  element in $B$ (and recall that $B$ is a subset of $\mathbb{F}_p$).
For each $i\in [k']$, the data distribution $\D_i$ is constructed as follows:
\begin{flushleft}\begin{itemize}
\item The distribution $\D_i$ is supported on $Y_{ (i-1)\alpha+1},\ldots,
  Y_{ i\alpha}$. 
 For each $j:(i-1)\alpha+1\le j\le i\alpha$,
     \\ $\D_i$ has probability $1/(2\alpha d)$ on each of the following $2d$ points in $Y_{j}$:
    $$ (1, j,1, z_{j,1}),\ldots,(1, j,2d,z_{j,2d}),$$ where $(z_{j,1},\ldots,z_{j,2d})$ come from
    $$(z_{j,1},\ldots, z_{j, 2d}) = \mathsf{RS}_{d, 2d}(y_{j,1}, \ldots, y_{j, d})\in \mathbb{F}_p^{2d}.$$
    Note that there are $2\alpha d$ points so the probabilities sum to $1$. All these points are labelled $1$.
\end{itemize}\end{flushleft}
Similarly Bob uses $f_B:[n]\rightarrow [n]$ to create $k'$ data distributions
  $\D_{k'+1},\ldots,\D_{2k'}$; the only difference is that these distributions 
  are supported on $Z$.
We prove in the following claim that these $2k'$ data distributions are
  consistent with a function $h:\X\rightarrow \{0,1\}$ in the hypothesis class $\mH$.

\begin{claim}
For any $f_A:[n]\rightarrow [n]$ and $f_B:[n]\rightarrow [n]$,
  there is a function $h$ in $\calH$ that is consistent with all $2k'$ 
  data distributions $\D_1,\ldots,\D_{2k'}$.
\end{claim}
\begin{proof}
Let $w^{(1)},\ldots,w^{(2c+2)}$ be defined in Definition \ref{def:pointer} using $f_A$ and $f_B$. Consider the function $h \in \mH$ specified by the tuple $(a_{i, j} \in [B]: i \in [2c + 1], j\in [d])$, where $a_{i, j} = y_{\rho^{-1}(w^{(i)})_1, j}$ for $i \in [2c+1], j \in [d]$. It is easy to verify that $h$ is consistent with $\D_1, \ldots, \D_{2k'}$.
\end{proof}

After creating these data distributions (with no communication so far),
  Alice and Bob simulate a $c$-pass $(\eps,0.01)$-continual learner 
  for $\mH$ over $\D_1,\ldots,\D_{2k'}$. 
If the $c$-pass learner uses $m$-bit memory,  the simulation can be
  done by a $(2c-1)$-round protocol with total communication $O(cm)$.
At the end of the simulation, Bob recovers from the learner
  a function $h_A\in \mH$ for some tuple $A=(a_{i,j}\in B:i\in [2c+1],j\in [d])$.
With probability at least $0.99$, we have 
\begin{equation}\label{eq:hehe5}
\ell_{\D_i}(h_A)\le \eps,\quad\text{for all $i\in [2k']$.}
\end{equation}
Let $q^{(1)},\ldots,q^{(2c+2)}\in Q$ be the sequence of $2c+2$ pointers
  obtained by following the procedure on $A$ described in the definition of the hypothesis class $\mH$.
The following claim shows that whenever $h_A$ satisfies (\ref{eq:hehe5}),   we  must have $\rho(q^{(\tau)})=w^{(\tau)}$ for all $\tau\in [2c+2]$, where 
  $w^{(1)},\ldots,w^{(2c+2)}$ are defined in Definition \ref{def:pointer} using $f_A$ and $f_B$.
As a result, to end the protocol, Bob just follows the procedure to compute
  $q^{(1)},\ldots,q^{(2c+2)}$ and output $\rho(q^{(2c+2)}) \pmod 2$.
This gives a $(2c-1)$-round communication protocol for the pointer chasing problem with success probability $0.99$.
This finishes the proof.\end{proof}

\begin{claim}
When $h_A$ satisfies (\ref{eq:hehe5}), we have $\rho(q^{(\tau)})=w^{(\tau)}$ and $$
y_{q^{(\tau-1)}_1,j}=\begin{cases}
  a_{I(q^{(\tau-1)}_1),j} & \text{for even $\tau$}\\[0.5ex]
  a_{J(q^{(\tau-1)}_1),j} & \text{for odd $\tau$}
  \end{cases}$$
for all $\tau=2,\ldots,2c+2$ and $j\in [d]$.
\end{claim}
\begin{proof}
We prove by induction on $\tau=2,\ldots,2c+2$.
We start with the base case when $\tau=2$.
Note that $\rho(q^{(1)})=1=w^{(1)}$ since $\rho(1,1,1)=1$. 
We also have $\smash{I(q^{(1)}_1)=1}$.
Assume for a contradiction  
$$
(a_{1,1},\ldots,a_{1,d})\ne (y_{1,1},\ldots,y_{1,d}).
$$
Then it follows from the error correction guarantee of Reed-Solomon code that
  at least $d$ entries of
$$
\RS_{d,2d}(a_{1,1},\ldots,a_{1,d})\quad\text{and}\quad
\RS_{d,2d}(y_{1,1},\ldots,y_{1,d})
$$
are different.
This implies that $\D_1$ is inconsistent with $h_A$ on at least 
  $d$ points and thus,
$$
\ell_{\D_1}(h_A)\ge \frac{d}{3\alpha d}\ge \frac{1}{2\alpha}>\eps,
$$
a contradiction with (\ref{eq:hehe5}).
So $(a_{1,1},\ldots,a_{1,d})= (y_{1,1},\ldots,y_{1,d})$.
To prove $\rho(q^{(2)})=w^{(2)}$, note that 
$$q^{(2)}=a_{1,1,1}=y_{1,1,1}=\rho^{-1}(f_A(\rho(1,1,1)))
=\rho^{-1}(f_A(1))=\rho^{-1}\big(w^{(2)}\big).$$
This finishes the proof of the base case.

The induction step is similar. Assume that the statement holds
  for all $2,\ldots,\tau-1$ and we now prove it holds for $\tau\le 2c+2$.
Assume $\tau$ is even; the proof for odd $\tau$ is symmetric.
Consider two cases. If in the $\tau$-th loop of the procedure, $\smash{t=I(q_1^{(\tau-1)})}\ne \nil$,
  then $t<\tau$ and $\smash{q^{(\tau-1)}_1=q^{(t)}_1.}$
Hence 
\begin{equation}\label{eq:hehe6}
y_{q_1^{(\tau-1)},j}=a_{I(q_1^{(\tau- 1)}),j}
\end{equation}
for all $j\in [d]$ follows from the inductive hypothesis.
If in the $\tau$-th loop, $\smash{I(q_1^{(\tau-1)})=\nil}$, then 
    we set $\smash{I(q_1^{(\tau-1)})=\tau-1}$ at the end of this loop and 
    (\ref{eq:hehe6}) follows from 
  an argument similar to the base case using the  
  error correction guarantee of Reed-Solomon code.
So (\ref{eq:hehe5}) holds in both cases.

In both cases, let $t\le \tau-1$ be the final value of $\smash{I(q_1^{(\tau-1)})}$.  Then
$$
q^{(\tau)}=a_{t,q^{(\tau-1)}_2,q^{(\tau-1)}_3}
=y_{q^{(\tau-1)}_1,q^{(\tau-1)}_2,q^{(\tau-1)}_3}
=\rho^{-1}\Big(f_A\big(\rho(q^{(\tau-1)})\big)\Big)=\rho^{-1}\big(w^{(\tau)}\big),
$$
where the last step used the inductive hypothesis  $\rho(q^{(\tau-1)})=w^{(\tau-1)}$.
This finishes the proof.
\end{proof}

\section{Discussion}
\label{sec:dis}
The problem of continual, or lifelong, learning is a major and crucial open challenge for Machine Learning, a key roadblock in the field's quest to transcend the stereotype of individual specialized tasks and make small steps towards brain-like learning: robust, unsupervised, self-motivated, embodied in a sensory-motor apparatus and embedded in the world, where data collection is deliberate and the sum total of the animal's experience is applied to each new task.  There seems to be a consensus in the field that the chief problem in extending the practical successes of Machine Learning in this direction is memory.

We formulated the problem of continual learning as a sequence of $k$ PAC learning tasks followed by a test, and showed a devastating lower bound for memory: the memory requirements for solving the problem are essentially increased from that of individual tasks by a factor of $k$. That is, unless a new approach is discovered --- or the PAC formalism is for some fundamental reason inadequate to model this variant of learning --- continual learning is impossible, with respect to memory requirements, in the worst case.

It has been argued that worst-case lower bounds do not always predict the difficulty of making practical progress in computational problems, and that this is especially true in the context of Machine Learning.  However, such lower bounds have a way of identifying the aspects of the problem that must be attended to in order to make progress, and pointing to alternative formulations that are more promising.  We feel that a comprehensive understanding of the theoretical difficulties of continual learning is a prerequisite for making progress in this important front.

Our multi-pass approach and the ensuing upper and lower bounds were admittedly inspired by streaming --- an algorithmic domain of a very different nature than continual learning.  However, it does provide a nice demonstration of the power of improper learning.  
We believe that the multi-pass MWU-based algorithm developed here may point the way to new empirical approaches to continual learning, perhaps involving the periodic replay of past data, as well as competition between several evolving variants of the learning device mediated by boosting.  Naturally, experiments will be needed to explore this direction.

\clearpage
\newpage
\bibliographystyle{alpha}
\bibliography{ref}

\newpage
\appendix
\section{Missing proof from Section \ref{sec:pre}}
\label{sec:app-pre}
\begin{proof}[Proof of Fact \ref{fact:mutual-general}]
One has
\begin{align*}
    \sum_{i=1}^{n}I(A_i; B) =&~ \sum_{i=1}^{n}H(A_i) - H(A_i | B) \\
    \leq &~ \sum_{i=1}^{n}H(A_i) - H(A_i | A_1, \ldots, A_{i-1}, B)\\
    = &~ \sum_{i=1}^{n}H(A_i) - H(A_i | A_1, \ldots, A_{i-1})\\
    &~+ \sum_{i=1}^{n} H(A_i | A_1, \ldots, A_{i-1}) - H(A_i | A_1, \ldots, A_{i-1}, B)\\
    = &~\sum_{i=1}^{n}H(A_i) - H(A_1, \ldots, A_{n}) + \sum_{i=1}^{n}I(A_i; B | A_1, \ldots, A_{i-1})\\
    = &~\sum_{i=1}^{n}H(A_i) - H(A_1, \ldots, A_{n}) + I(A_1, \ldots, A_n; B).
    \end{align*}
This concludes the proof.
\end{proof}

\section{Missing proof from Section \ref{sec:algo}}
\label{sec:algo-app}

First, we provide the proof of Lemma \ref{lem:estimate-quantile}.
\begin{proof}[Proof of Lemma \ref{lem:estimate-quantile}]
Fix an index $i \in [k]$ and a round $t \in [c]$, we prove the first and second claim by induction.
The base case $\tau = 0$ holds trivially and suppose it continues to hold up to time $\tau-1$, then
\begin{align*}
    \E[|\bar{S}_{i, t, \tau}^{q}|] = \E[|S_{i,t, \tau}^{q} \cap \X_{i, t, \tau}|] = M_1\cdot \Pr_{(x, y)\sim \D_i}[x\in \X_{i, t, \tau}] \geq (1 - (1 +\gamma)\eps_t)^{t-1}M_1 \geq \frac{3}{4}M_1,
\end{align*}
where the first step follows from the definition of $\bar{S}_{i, t, \tau}^{q}$, the second step follows from the linearity of expectation, the third step holds due to the inductive hypothesis, and the last step holds as $(1-(1+\gamma)\eps_t)^{t-1} \geq 1 - (1 +\gamma)\eps_t\cdot (t-1) \geq \frac{3}{4}$.

Therefore, by Chernoff bound, one has
\begin{align}
    \Pr_{S^{q}_{i, t, \tau}\sim \D_{i}^{M_1}}[|S^{q}_{i, t, \tau} \cap \X_{i, t, \tau}| < M_1/2] \leq \exp(-M_1/36) \leq \frac{\delta}{60kc^2} \label{eq:q3}
\end{align}
Now we condition on $|\bar{S}_{i, t, \tau}^{q}| = |S^{q}_{i, t, \tau} \cap \X_{i, t, \tau}| \geq M_1/2$. For any $\lambda \in [0,1]$, define the exact $\lambda$-quantile of $\{(x, \exp(\eta\sum_{\nu=1}^{\tau}\mathbf{1}[h_{\nu}(x)\neq y]))\}_{(x, y)\sim \D_{i} | x\in \X_{i, t, \tau}}$ to be $(x_{i, t, \tau, \lambda}, q_{i, t, \tau, \lambda})$.
The {\sc EstimateQuantile} returns the $\eps_t$-quantile of the empirical sample, then we have that
\begin{align}
    \Pr\left[(x_{i, t, \tau}, \hat{q}_{i,t, \tau}) \succeq (x_{i, t, \tau, (1+\gamma)\eps_t}, q_{i, t, \tau, (1+\gamma)\eps_t})\right] \leq \exp(-\eps_t\gamma^2 M_1/24) \leq   \frac{\delta}{60kc^2}.\label{eq:q4}
\end{align}
where the second step holds by Chernoff bound the fact that tie breaking only affects $o(\eps_0/c^3)$ quantile.
Similarly, 
\begin{align}
    \Pr\left[(x_{i, t, \tau}, \hat{q}_{i,t, \tau}) \preceq (x_{i, t, \tau, (1-\gamma)\eps_t}, q_{i, t, \tau, (1-\gamma)\eps_t})\right] \leq \exp(-\eps_0\gamma^2 M_1/6) \leq   \frac{\delta}{60kc^2} \label{eq:q5}.
\end{align}
Combining Eq.~\eqref{eq:q4}, Eq.~\eqref{eq:q5} and an union bound, we can conclude the proof the first claim. Note that 
\[
\X_{i, t, \tau} \backslash \X_{i, t, \tau+1} = \left\{x \in \X_{i, t, \tau}: (x, \exp(\eta\sum_{\tau=1}^{\tau} \mathbf{1}\{h_{\tau}(x)\neq y\})) \succeq (x_{i, t, \tau}, \hat{q}_{i, t, \tau}) \right\},
\]
and therefore, 
\[
P_{i, t, \tau+1} \in \left[(1-(1+\gamma)\eps_t)\cdot P_{i, t, \tau}, (1-(1-\gamma)\eps_t) \cdot P_{i, t,\tau}\right].
\]
Hence completing the proof of second claim.
\end{proof}

We then prove
\begin{proof}[Proof of Lemma \ref{lem:hier}]
The first claim holds due to definition and therefore we focus on the second claim.
We prove by induction on $\tau$. 
Fix an index $i \in [k]$, the case of $\tau = 1$ holds trivially, as $\X_{i, 1, 1} = \cdots = \X_{i,t, 1} = \X$ by definition. 
Suppose the induction holds up to time $\tau-1$, then we prove $\X_{i, t + 1, \tau} \subseteq \X_{i, t, \tau}$ holds for any $t \in [\tau: c-1]$.
By definition, we have (i) $\X_{i, t+1, \tau} \subseteq \X_{i, t+1, \tau-1}$ and by induction, (ii) $\X_{i, t+1, \tau-1} \subseteq \X_{i, t, \tau-1}$.
Condition on the event of Lemma \ref{lem:estimate-quantile}, it suffices to prove for any data point $x \in \X_{i, t+1, \tau}$, it does not belong to the top $(1+\gamma)\eps_{t}$-quantile of $\X_{i, t, \tau-1}$.

First of all, we know data points in $\X_{i, t+1, \tau-1} \backslash\X_{i,t+1, \tau}$ has decent amount of probability, i.e.
\begin{align*}
    \Pr_{(x, y) \sim \D_i}[x \in \X_{i, t+1, \tau-1} \backslash\X_{i,t+1, \tau}] \geq (1-\gamma)\eps_{t+1} \cdot P_{i, t+1, \tau-1}.
\end{align*}
Combining (i) and (ii), these data points belong to $\X_{i, t,\tau-1}$, i.e.,
\[
\X_{i, t+1, \tau-1} \backslash\X_{i,t+1, \tau} \subseteq \X_{i, t,\tau-1}.
\]
Hence it suffices to prove data points in $\X_{i, t+1, \tau-1} \backslash\X_{i,t+1, \tau}$ take up at least $(1 + \gamma)\eps_{t}$ portion of $\X_{i, t, \tau-1}$ (again, assuming the event of Lemma \ref{lem:estimate-quantile} holds), i.e.,
\begin{align}
    (1-\gamma)\eps_{t+1} \cdot P_{i, t+1, \tau-1} \geq (1+\gamma)\eps_{t} P_{i, t, \tau-1}. \label{eq:q10}
\end{align}
This is true as
\begin{align*}
    \frac{(1 - \gamma)\eps_{t+1}}{(1+ \gamma) \eps_{t}} \geq \left(\frac{1-(1-\gamma)\eps_{t}}{1-(1+\gamma)\eps_{t+1}}\right)^{c} \geq \left(\frac{1-(1-\gamma)\eps_{t}}{1-(1+\gamma)\eps_{t+1}}\right)^{\tau-1} \geq \frac{P_{i, t, \tau-1}}{P_{i, t+1, \tau-1}},
\end{align*}
where the first step holds due to the choice of parameter (see Lemma \ref{lem:tech-algo}), the third step holds due to Lemma \ref{lem:estimate-quantile}.
Hence, we have proved Eq.~\eqref{eq:q10} and concludes the proof.
\end{proof}

Next, we prove
\begin{proof}[Proof of Lemma \ref{lem:estimate-weight}]

For any $i\in [k], t\in [c]$, according to the Definition \ref{def:trun}:
\begin{align*}
    w_{i, t} = &~ \sum_{x\in \X}\mathbf{1}[x \in \X_{i, t, t-1}]\cdot \mu_{\D_i}(x) \cdot \min\left\{\exp(\eta \sum_{\tau=1}^{t-1}\mathbf{1}[h_\tau(x) \neq y]), \hat{q}_{i, t,t-1}\right\}\\
    = &~ \E_{x\sim \D_{i}}\left[\mathbf{1}[x \in \X_{i, t, t-1}]\cdot\min\left\{\exp(\eta \sum_{\tau=1}^{t-1}\mathbf{1}[h_\tau(x) \neq y]), \hat{q}_{i, t, t-1}\right\}\right].
\end{align*}
Note each term 
\begin{align}
\mathbf{1}[x \in \X_{i, t, t-1}]\cdot \min\left\{\exp(\eta \sum_{\tau=1}^{t-1}\mathbf{1}[h_\tau(x) \neq y]), \hat{q}_{i, t,t-1}\right\} \in [0, \hat{q}_{i, t, t-1}], \label{eq:w2}
\end{align} 
condition on the event of Lemma \ref{lem:estimate-quantile}, the expectation satisfies
\begin{align}
&~ \E_{x\sim \D_{i}}\left[\mathbf{1}[x \in \X_{i, t, t-1}]\cdot\min\left\{\exp(\eta \sum_{\tau=1}^{t-1}\mathbf{1}[h_\tau(x) \neq y]), \hat{q}_{i, t, t-1}\right\}\right]\notag \\
\geq &~ 
\E_{x\sim \D_{i}}\left[\mathbf{1}[x \in \X_{i, t, t-1} \backslash \X_{i, t,t}]\cdot\min\left\{\exp(\eta \sum_{\tau=1}^{t-1}\mathbf{1}[h_\tau(x) \neq y]), \hat{q}_{i, t, t-1}\right\}\right] \notag \\
= &~ \E_{x\sim \D_{i}}\mathbf{1}[x \in \X_{i, t,t-1} \backslash \X_{i,t,t}]\cdot\hat{q}_{i, t,t-1}\notag  \\
= &~ (P_{i, t, t-1} - P_{i, t, t}) \cdot \hat{q}_{i, t, t-1}\notag \\
\geq &~ (1-\gamma)\eps_t P_{i, t, t-1} \hat{q}_{i, t, t-1} \geq \frac{1}{2}\eps_t \hat{q}_{i, t,t-1}.\label{eq:w3}
\end{align}
The second step holds as every $x\in \X_{i, t,t-1} \backslash \X_{i, t,t}$ satisfies $\exp(\eta \sum_{\tau=1}^{t-1}\mathbf{1}[h_\tau(x) \neq y]) \geq \hat{q}_{i, t,t-1}$ (see the definition at Eq.~\eqref{eq:xit}), the third step holds due to the definition of $P_{i, t,t-1}$ and $P_{i, t,t}$, the fourth step holds due to Lemma \ref{lem:estimate-quantile} and the last holds as $(1-\gamma)P_{i, t, t-1} \geq (1-\gamma)(1-(1+\gamma)\eps_t)^{t-1} \geq \frac{1}{2}$.

Recall in the procedure of {\sc EstimateWeight}, we draw $M_2 = \Omega(\log(kc/\delta)/\eps_0\alpha^2)$ samples from $\D_i$ to estimate $w_{i, t}$, hence we have
\begin{align*}
&~ \Pr\left[|\hat{w}_{i, t} - w_{i, t}| \geq \frac{\alpha}{8}w_{i, t}\right]\\
= &~ \Pr\left[\left|\frac{1}{|M_2|}\sum_{(x, y)\in S_{i, t}^{w}\cap \X_{i, t,t-1}} \min\{\exp(\eta \sum_{\tau=1}^{t-1}\mathbf{1}[h_\tau(x) \neq y]), \hat{q}_{i, t, t-1}\} - w_{i, t}\right| \geq \frac{\alpha}{8}w_{i, t}\right]\\
\leq &~ 2\exp(-M_2 \alpha^2 \eps_t / 384) \leq \frac{\delta}{20kc}.
\end{align*}
where the first step holds due to the definition of $\bar{S}_{i, t}^{w}$, the second step holds due to Chernoff bound, Eq.~\eqref{eq:w2} and Eq.~\eqref{eq:w3}, the last step follows from the choice $M_2$.

Using an union bound over all $i \in [k]$, we further have
\begin{align*}
    \hat{p}_{i, t} = \frac{\hat{w}_{i, t}}{\sum_{i'=1}^{k}\hat{w}_{i', t}} \in (1 \pm \frac{\alpha}{3}) \frac{w_{i, t}}{\sum_{i'=1}^{k}w_{i', t}} = (1 \pm \frac{\alpha}{3}) p_{i, t}
\end{align*}
holds with probability at least $1 - \frac{\delta}{20c}$.
We finish the proof by applying an union bound over $t\in [c]$.
\end{proof}

We next analyse the {\sc TruncatedRejectionSampling} procedure.

\begin{proof}[Proof of Lemma \ref{lem:truncated-rejection-sampling}]
The first claim follows directly from the {\sc TruncatedRejectionSampling} procedure. In particular, it rejects element from $\X\backslash\X_{i, t,t-1}$, and for element $ x\in \X_{i, t,t-1}$, it accepts with probability $\min\{\exp(\eta\sum_{\nu=1}^{t-1}\mathbf{1}[h_{\nu}(x)\neq y]), \hat{q}_{i, t,t-1}\}/\hat{q}_{i, t,t-1}$.

For the second claim, condition on the event of Lemma \ref{lem:estimate-quantile}, we first bound the expected overhead of {\sc TruncatedRejectionSampling}.
Note {\sc TruncatedRejectionSampling} would not reject element in $X_{i, t,t-1}\backslash X_{i, t,t}$ and
\[
\Pr_{x\sim \D_{i}}[x \in \X_{i, t,t}\backslash \X_{i, t,t-1}] \geq (1-\gamma)\eps_t P_{i, t, t-1} \geq (1-\gamma)(1-(1+\gamma)\eps_t)^{t-1}\eps_t \geq \frac{1}{2}\eps_0 ,
\]
where the first two steps hold due to Lemma \ref{lem:estimate-quantile} and the last step holds due to the choice of $\eps_t$. Hence, the overhead from rejection sampling is at most $2/\eps_0$ (in expectation).

Fix a round $t\in [c]$, for any $i \in [k]$, the expected number of samples drawn from $\hat{D}_{i, t}$ equals $n_{i, t} := \frac{\hat{w}_{i, t}}{\sum_{i'\leq i}\hat{w}_{i', t}}\cdot N$ and it satisfies
\begin{align*}
    \sum_{i=1}^{k}n_{i, t} = &~  N \sum_{i=1}^{k} \frac{\hat{w}_{i, t}}{\sum_{i'\leq i}\hat{w}_{i', t}} \leq  N + N\sum_{i=2}^{k} \log \frac{\sum_{i'\leq i}\hat{w}_{i', t}}{\sum_{i'\leq i -1}\hat{w}_{i', t}}\\
    \leq &~ \log \frac{\sum_{i\in [k]}\hat{w}_{i, t}}{\hat{w}_{1, t}} \leq  O(N \log(k/\eps))
\end{align*}
where the first step follows from the definition, the second step follows from $\frac{a-b}{a} \leq \log \frac{a}{b}$, the last step holds as $\hat{w}_{1, t} \geq P_{1, t, t-1} \geq \frac{1}{2}$ (Lemma \ref{lem:estimate-quantile}) and 
\[
\hat{w}_{i, t} \leq \exp(\eta \sum_{\tau=1}^{t-1}\mathbf{1}[h_\tau(x) \neq y]) \leq \exp((c-1)\cdot \eta) = O\left(\frac{k^2}{\eps^2}\right)
\]
holds for any $i \in [k]$.

Via a simple application of Chernoff bound, one can prove with probability at least $1 - \frac{\delta}{20c}$, (1) the total number of sample drawn from the truncated distribution $\{\D_{i, t, \trun}\}_{i \in [k]}$ is at most $O(N \log(kdc/\eps\delta))$, (2) the overhead of {\sc TruncatedRejectionSampling} procedure never exceeds $O(c\log(kdc/\eps\delta)/\eps)$.
Using an union bound, the total number of sample being drawn for the training set $\{S_{t}\}_{t\in [c]}$ are bounded by
\[
O(N \log(kdc/\eps\delta)) \cdot O(c\log(kdc/\eps\delta)/\eps) \cdot c \leq O\left(\frac{dc^2\log^3(kdc/\eps\delta)}{\eps\alpha}\right).
\]
We conclude the proof here.
\end{proof}

Next, we prove Lemma \ref{lem:training-set} and Lemma \ref{lem:acc}.

\begin{proof}[Proof of Lemma \ref{lem:training-set}]
For any $t\in [c]$, and for any training sample preserved at the end, it comes from the $i$-th task with probability $\frac{\hat{w}_{i, t}}{\sum_{i'\leq i}\hat{w}_{i', t}}\cdot \prod_{i'\geq i+1}\frac{\sum_{i''< i'}\hat{w}_{i'', t}}{\sum_{i''\leq i}\hat{w}_{i'', t}} = \frac{\hat{w}_{i, t}}{\sum_{i'\in [k]}\hat{w}_{i',t}} = \hat{p}_{i, t}$. For any training sample comes from the $i$-th task, by Lemma \ref{lem:truncated-rejection-sampling}, it follows from the distribution $\D_{i, t, \trun}$.
\end{proof}

\begin{proof}[Proof of \ref{lem:acc}]
The algorithm draws $N = O(\frac{d + \log(kc/\delta)}{\alpha})$ samples from $\hat{\D}_{t, \trun}$ and runs ERM on it. Due to the realizable assumption, with probability at least $1 - \frac{\delta}{20c}$, the output hypothesis satisfies
\begin{align*}
\ell_{\D_{t, \trun}}(h_t) \leq \ell_{\hat{\D}_{t, \trun}}(h_t) + \frac{\alpha}{3} \leq \frac{\alpha}{2},
\end{align*}
where the first step follows from Lemma \ref{lem:distribution-tv}, the seconds step holds due to the VC theory (see Lemma \ref{lem:learning-erm})
\end{proof}

\begin{lemma}[Technical lemma]
\label{lem:tech-algo}
Let $\eps \in (0, 1/10), c\geq 1$, $\gamma = \frac{1}{10c^2}$, $\eps_t = (1 + \frac{1}{c})^{t}\cdot \frac{\eps}{20c}$, then
\begin{align*}
    \frac{(1 - \gamma)\eps_{t+1}}{(1+ \gamma) \eps_{t}} \geq \left(\frac{1-(1-\gamma)\eps_{t}}{1-(1+\gamma)\eps_{t+1}}\right)^{c}.
\end{align*}
\end{lemma}
\begin{proof}
The RHS satisfies
\begin{align*}
    \left(\frac{1-(1-\gamma)\eps_{t}}{1-(1+\gamma)\eps_{t+1}}\right)^{c} =&~ \left(1 + \frac{\gamma \eps_{t+1} + \gamma \eps_{t} + \eps_{t+1} - \eps_t}{1-(1+\gamma)\eps_{t}}\right)^{c}\\
    \leq &~ \left(1 + \frac{\eps}{c^2}\right)^{c}\\
    \leq &~ \exp(\frac{\eps}{c}) \leq 1 + \frac{2\eps}{c}
\end{align*}
The second step holds as $\gamma \eps_{t+1} \leq \frac{\eps}{30c^3}, \gamma \eps_{t} \leq \frac{\eps}{30c^3}$, $\eps_{t+1} - \eps_{t} \leq \frac{\eps}{3c^2}$ and $1 - (1+\gamma)\eps_{t+1}\geq 1/2$, the second step follows from $1+x\leq \exp(x)$, the third step follows from $\exp(x) \leq 1 + 2x$ when $x \in (0, 1/10)$.

Meanwhile, the LHS satisfies
\begin{align*}
    \frac{(1 - \gamma)\eps_{t+1}}{(1+ \gamma) \eps_{t}} = (1 + \frac{1}{c})\cdot\frac{10c^2 - 1}{10c^2 + 1} \geq 1 + \frac{2\eps}{c}.
\end{align*}
We complete the proof.
\end{proof}

\end{document}